\documentclass[11pt,letterpaper]{article}

\usepackage[top=1in, bottom=1in, left=1in, right=1in]{geometry}

\usepackage{microtype}
\usepackage{graphicx}
\usepackage{booktabs} 

\usepackage{hyperref}

\usepackage{amsmath}
\usepackage{amsthm}
\usepackage{amssymb}
\usepackage{mathabx}
\usepackage{bm}
\usepackage{graphicx}
\usepackage{xcolor}
\usepackage{tikz}
\usepackage{subcaption}
\usepackage{enumitem}

\usepackage{algorithm2e}
\usepackage{algorithmic}
\usepackage{cite}

\newcommand{\real}{\mathbb{R}}
\newcommand{\T}{\intercal}

\newcommand{\seq}[1]{[#1]}

\newcommand{\prob}{\mathbb{P}}
\newcommand{\E}{\mathbb{E}}
\newcommand{\calO}{\mathcal{O}}

\newcommand{\sym}{\mathbb{S}^d_+}
\newcommand{\bX}{\mathbf{X}}
\newcommand{\bZ}{\mathbf{Z}}
\newcommand{\bI}{\mathbf{I}}
\newcommand{\bM}{\mathbf{M}}
\newcommand{\by}{\mathbf{y}}
\newcommand{\bz}{\mathbf{z}}
\newcommand{\be}{\mathbf{e}}
\newcommand{\bg}{\mathbf{g}}
\newcommand{\bH}{\mathbf{H}}

\newcommand{\bHhat}{\hat{\mathbf{H}}}
\newcommand{\bzero}{\mathbf{0}}
\newcommand{\balpha}{\mathbf{\mu}}
\newcommand{\tw}{\tilde{w}}
\newcommand{\tZ}{\tilde{\bZ}}
\newcommand{\inner}[2]{\langle #1, #2 \rangle}
\newcommand{\diag}{{\rm{diag}}}
\newcommand{\eig}{{\rm{eig}}}
\newcommand{\diff}[2]{\frac{\partial{#1}}{\partial{#2}}}
\newcommand{\inertia}{{\rm{In}}}
\newcommand{\tr}{{\rm{trace}}}
\newtheorem{theorem}{Theorem}
\newtheorem{corollary}{Corollary}
\newtheorem{lemma}{Lemma}
\newtheorem{assumption}{Assumption}
\newtheorem{proposition}{Proposition}
\newtheorem{definition}{Definition}

\title{Fair Sparse Regression with Clustering: An Invex Relaxation for a Combinatorial Problem}
\date{}

\author{%
	Adarsh Barik \\
	Department of Computer Science\\
	Purdue University\\
	West Lafayette, Indiana, USA\\
	\texttt{abarik@purdue.edu} \\
	\and
	Jean Honorio \\
	Department of Computer Science \\
	Purdue University \\
	West Lafayette, Indiana, USA\\
	\texttt{jhonorio@purdue.edu} \\
}

\begin{document}

\maketitle

\begin{abstract}
	In this paper, we study the problem of fair sparse regression on a biased dataset where bias depends upon a hidden binary attribute. The presence of a hidden attribute adds an extra layer of complexity to the problem by combining sparse regression and clustering with unknown binary labels. The corresponding optimization problem is combinatorial, but we propose a novel relaxation of it as an \emph{invex} optimization problem. To the best of our knowledge, this is the first invex relaxation for a combinatorial problem. We show that the inclusion of the debiasing/fairness constraint in our model has no adverse effect on the performance. Rather, it enables the recovery of the hidden attribute. The support of our recovered regression parameter vector matches exactly with the true parameter vector. Moreover, we simultaneously solve the clustering problem by recovering the exact value of the hidden attribute for each sample. Our method uses carefully constructed primal dual witnesses to provide theoretical guarantees for the combinatorial problem. To that end, we show that the sample complexity of our method is logarithmic in terms of the dimension of the regression parameter vector.
\end{abstract}

\section{Introduction}
\label{sec:introduction}

In modern times, machine learning algorithms are used in a wide variety of applications, many of which are decision making processes such as hiring~\cite{hoffman2018discretion}, predicting human behavior~\cite{subrahmanian2017predicting}, COMPAS (Correctional Offender Management Profiling for Alternative Sanctions) risk assessment~\cite{brennan2009evaluating}, among others. These decisions have large impacts on society~\cite{kleinberg2018human}. Consequently, researchers have shown interest in developing methods that can mitigate unfair decisions and avoid bias amplification. Several fair algorithms have been proposed for machine learning problems such as regression~\cite{agarwal2019fair,berk2017convex,calders2013controlling},  classification~\cite{agarwal2018reductions,donini2018empirical,dwork2012fairness,feldman2015certifying,hardt2016equality,huang2019stable,pedreshi2008discrimination,zafar2019fairness,zemel2013learning} and clustering~\cite{backurs2019scalable,bera2019fair,chen2019proportionally,chierichetti2017fair,huang2019coresets}. A common thread in the above literature is that performance is only viewed in terms of risks, e.g., misclassification rate, false positive rate, false negative rate, mean squared error.

In the literature, fairness is discussed in the context of discrimination based on membership to a particular group (e.g. race, religion, gender) which is considered a sensitive attribute. Fairness is generally modeled explicitly by adding a fairness constraint or implicitly by incorporating it in the model itself. There have been several notions of fairness studied in linear regression. \cite{berk2017convex} proposed notions of individual fairness and group fairness, and modeled them as penalty functions. \cite{calders2013controlling} proposed the fairness notions of equal means and balanced residuals by modeling them as explicit constraints. \cite{agarwal2019fair}, \cite{fitzsimons2019general} and \cite{chzhen2020fair} studied demographic parity. While \cite{agarwal2019fair}, \cite{fitzsimons2019general} modeled it as an explicit constraint, \cite{chzhen2020fair} included it implicitly in their proposed model.

All the above work assume access to the sensitive attribute in the training samples and provide a framework which are inherently fair. Our work fundamentally differs from these work as we do not assume access to the sensitive attribute. Without knowing the sensitive attribute, it becomes difficult to ascertain bias, even for linear regression. In this work, we focus on identifying unfairly treated members/samples. This adds an extra layer of complexity to linear regression. We solve the linear regression problem while simultaneously solving a clustering problem where we identify two clusters -- one which is positively biased and the other which is negatively biased. Table~\ref{tab:fairness} shows a consolidated comparison of our work with the existing literature. 

Once one identifies bias (positive or negative) for each sample, one could perform debiasing which would lead to the fairness notion of equal means~\cite{calders2013controlling} among the two groups (See Figure \ref{fig:debias}). It should be noted that identifying groups with positive or negative bias may not be same as identifying the sensitive attribute. The reason is that there may be multiple attributes that are highly correlated with the sensitive attribute. In such a situation,
these correlated attributes can facilitate indirect discrimination even if the sensitive attribute is identified and removed. This is called the red-lining effect~\cite{calders2010three}. Our model avoids this by directly identifying biased groups. 	
\begin{table}
	\caption{\label{tab:fairness}Comparison to prior work. Notation: $s$ is the number of non-zero entries in the regression parameter vector and $d$ is its dimension. The terms independent of $s$ and $d$ are not shown in the order notation.}
	\begin{tabular}{p{4.5cm}lp{3cm}r}
		\toprule
		Paper & Hidden sensitive attribute & Modeling type & Sample complexity \\
		\midrule
\cite{calders2013controlling,agarwal2019fair,fitzsimons2019general} & No & Explicit constraint & Not provided \\
		\cite{berk2017convex} & No & Penalty function & Not provided \\
		\cite{chzhen2020fair} & No & Implicit & Not provided \\
		\textbf{Our paper} & \textbf{Yes} & \textbf{Implicit} & $\Omega(s^3 \log d)$
	\end{tabular}
\end{table}
\begin{figure*}[!ht]
	\centering
	\begin{subfigure}{.33\textwidth}
		\centering
		\includegraphics[width=\linewidth]{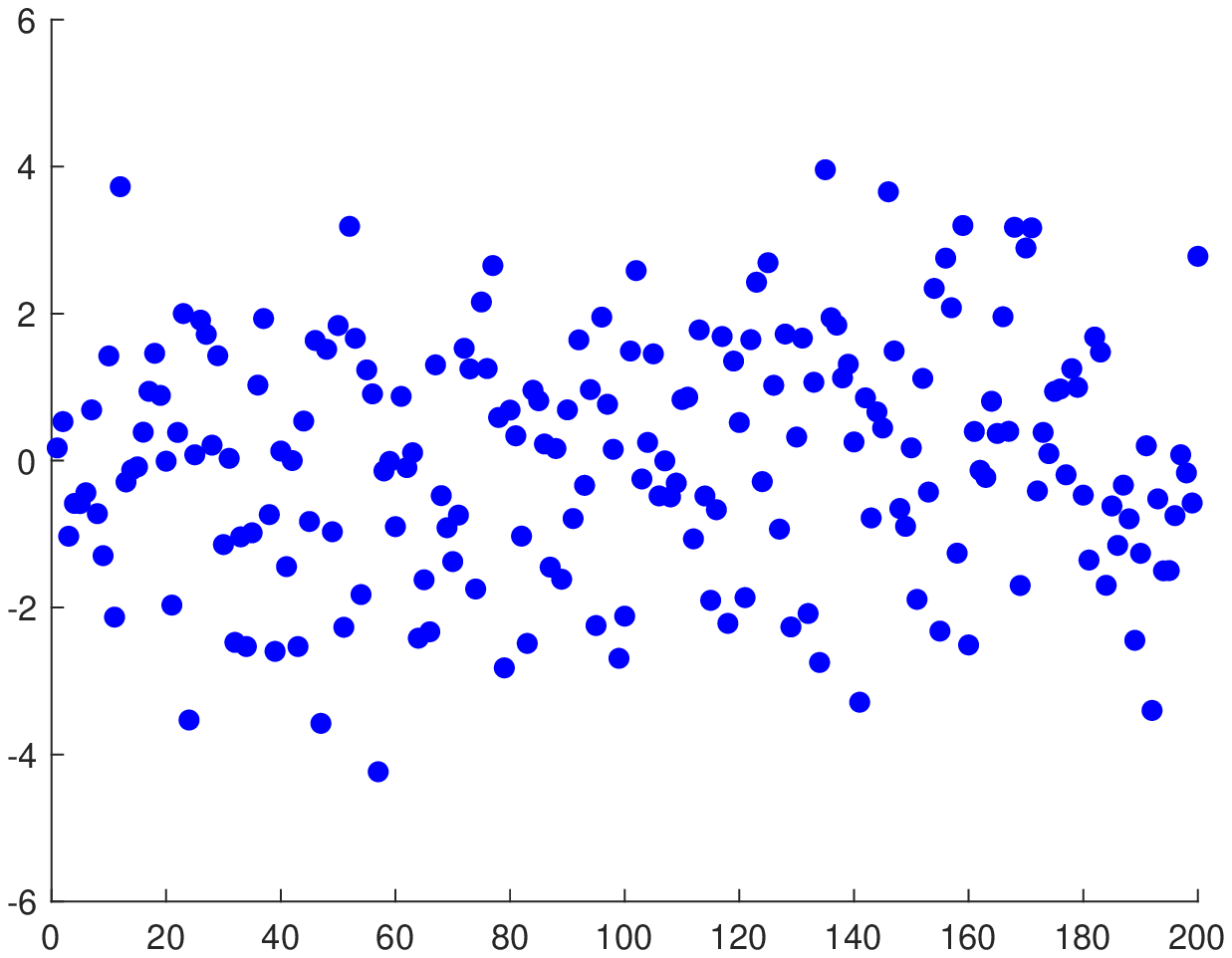}
		\caption{Biased data without hidden sensitive attribute}
	\end{subfigure}%
	\begin{subfigure}{.33\textwidth}
		\centering
		\includegraphics[width=\linewidth]{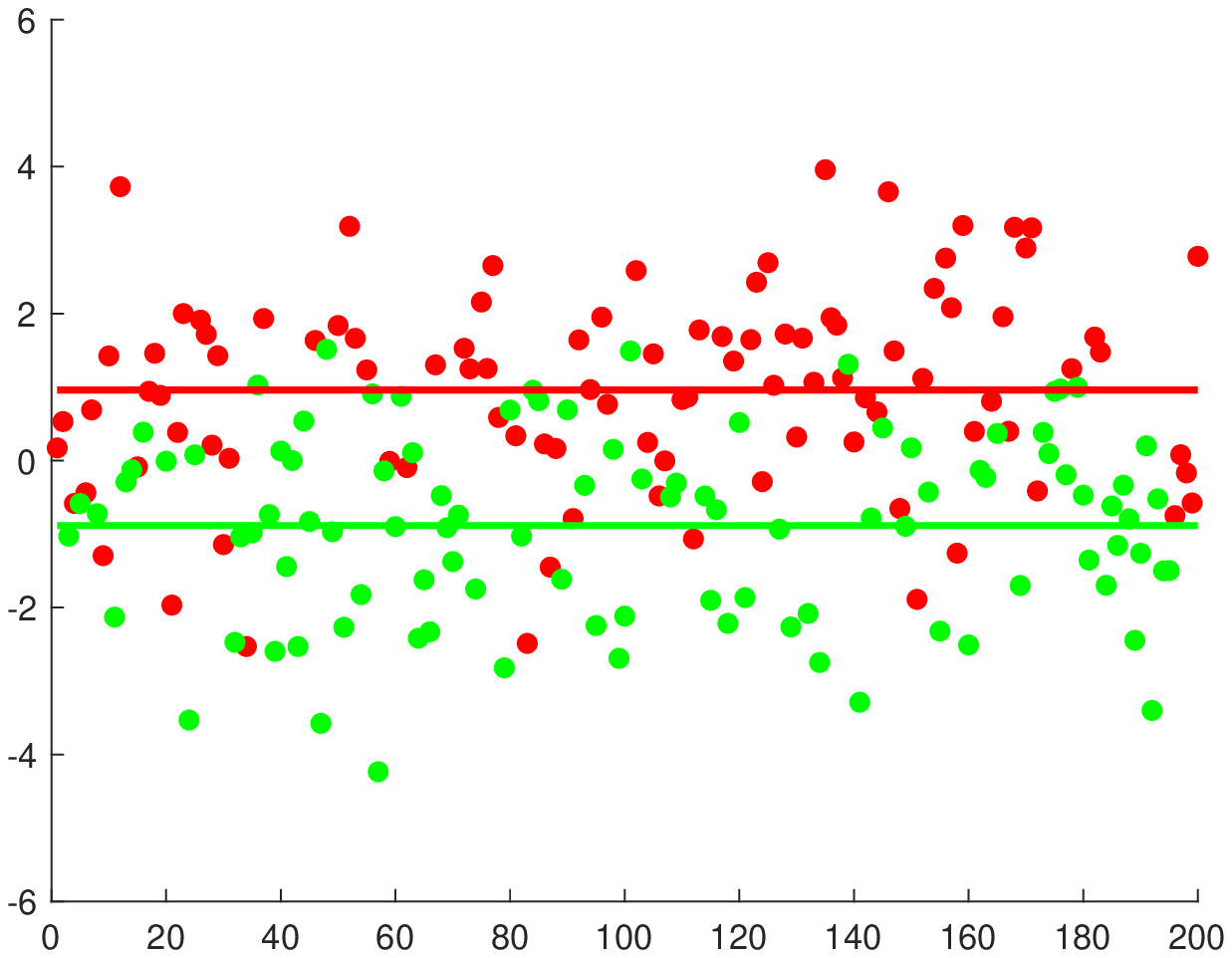}	
		\caption{Biased data after identifying hidden sensitive attribute}
	\end{subfigure}%
	\begin{subfigure}{.33\textwidth}
		\centering
		\includegraphics[width=\linewidth]{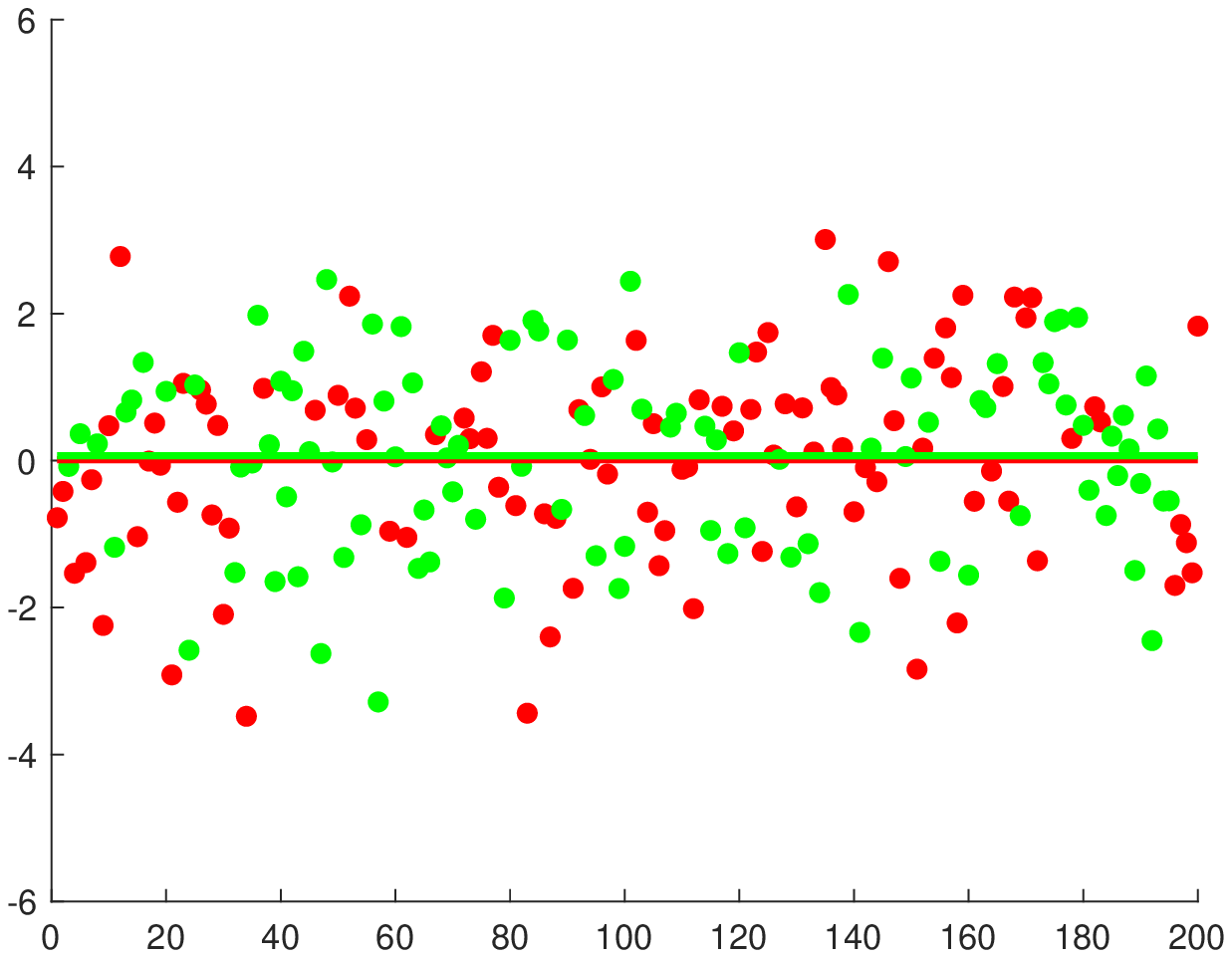}
		\caption{Data after debiasing}
	\end{subfigure}
	\caption{\label{fig:debias}Data before debiasing and after debiasing. Notice how means for two groups (shown as horizontal lines) become almost equal after debiasing.}
\end{figure*}

While the standard algorithms solving the sparse/LASSO problem in this setting do provide an estimate of the regression parameter vector, they do not fit the model accurately as they fail to consider any fairness criteria in their formulations. It is natural then to think about including the hidden attribute in LASSO itself. However, this breaks the convexity of the loss function which makes the problem intractable by the standard LASSO algorithms. The resulting problem is a combinatorial version of sparse linear regression with added clustering according to the hidden attribute. In this work, we propose a novel technique to tackle the combinatorial LASSO problem with a hidden attribute and provide theoretical guarantees about the quality of the solution given a sufficient number of samples. Our method provably detects unfairness in the system. It should be noted that observing unfairness does not always imply that the designer of the system intended for such inequalities to arise. In such cases, our method acts as a check to detect and remove such unintended discrimination. While the current belief is that there is a trade-off between fairness and performance~\cite{corbett2017algorithmic,kleinberg2017inherent,pleiss2017fairness,zliobaite2015relation,zhao2019inherent}, our theoretical and experimental results show evidence on the contrary. Our theoretical results allow for a new understanding of fairness, as an ``enabler'' instead of as a ``constraint''. 

\paragraph{Contribution.}

Broadly, we can categorize our contribution in the following points:
\begin{itemize}[noitemsep]
	\item \textbf{Defining the problem}: We formulate a novel combinatorial version of sparse linear regression which takes fairness/bias into the consideration. The addition of clustering comes at no extra cost in terms of the performance.
	\item \textbf{Invex relaxation}: Most of the current methods solve convex optimization problems as it makes the solution tractable. We propose a novel relaxation of the combinatorial problem and formally show that it is invex. To the best of our knowledge, this is the first invex relaxation for a combinatorial problem.  
	\item \textbf{Theoretical Guarantees}:  Our method can detect bias in the system. In particular, our method recovers the exact hidden attributes for each sample and thus provides an exact measure of bias between two different groups. Our method solves linear regression and clustering simultaneously with theoretical guarantees. To that end, we recover the true clusters (hidden attributes) and a regression parameter vector which is correct up to the sign of entries with respect to the true parameter vector. On a more technical side, we provide a primal-dual witness construction for our invex problem and provide theoretical guarantees for recovery. The sample complexity of our method varies logarithmically with respect to dimension of the regression parameter vector.     
\end{itemize}

\section{Notation and Problem Definition}
\label{sec:problem definition}

In this section, we collect all the notations used throughout the paper. We also formally introduce our novel problem. We consider a problem where we have a binary hidden attribute, and where fairness depends upon the hidden attribute. Let $y \in \real $ be the response variable and $X \in \real^d$ be the observed attributes. Let $z^* \in \{-1, 1\}$ be the \emph{hidden} attribute and $\gamma \in \real_{> 0}$ be the amount of bias due to the hidden attribute. The response $y$ is generated using the following mechanism:
\begin{align}
\label{eq:generative process}
\begin{split}
y = X^\T w^* + \gamma z^* + e  
\end{split}
\end{align}

where $e$ is an independent noise term. For example, $y$ could represent the market salary of a new candidate, $X$ could represent the candidate's
skills and $z$ could represent the population group the candidate belongs to (e.g., majority or minority). While the group of the candidate is not public knowledge, a bias associated with the candidate's group may be present in the underlying data. For our problem, we will assume that an estimate of the bias $\gamma \in \real_{>0}$ is available. In practice, even a rough estimate ($\pm 25\%$) of $\gamma$ also works well (See Appendix~\ref{appndx:recovery with gamma}).

Let $\seq{d}$ denote the set $\{1, 2, \cdots, d\}$. We assume $X \in \real^d$ to be a zero mean sub-Gaussian random vector~\cite{hsu2012tail} with covariance $\Sigma \in \sym$, i.e., there exists a $\rho > 0$, such that for all $\alpha \in \real^d$ the following holds: $	\E(\exp(\alpha^\T X)) \leq \exp(\frac{\| \alpha \|_2^2 \rho^2}{2}) $. By simply taking $\alpha_i = r$ and $\alpha_k = 0, \forall k \ne i$, it follows that each entry of $X$ is sub-Gaussian with parameter $\rho$. In particular, we will assume that $\forall i \in \seq{d}\, , \frac{X_i}{\sqrt{\Sigma_{ii}}}$ is a sub-Gaussian random variable with parameter $\sigma > 0$. It follows trivially that $\max_{i \in \seq{d}} \sqrt{\Sigma_{ii}} \sigma \leq \rho$.  We will further assume that $e$ is zero mean independent sub-Gaussian noise with variance $\sigma_e$. We assume that as the number of samples increases, the noise in the model gently decreases. We model this by taking $\sigma_e = \frac{k}{\sqrt{\log n}}$ for some $k > 0$. Our setting works with a variety of random variables as the class of sub-Gaussian random variable includes for instance Gaussian variables, any bounded random variable (e.g., Bernoulli, multinomial, uniform), any random variable with strictly log-concave density, and any finite mixture of sub-Gaussian variables. Notice that for the group with $z = +1$, $\E(y) = \gamma$ and for the group with $z = -1$, $\E(y) = -\gamma$. This means that after correctly identifying groups, one could perform debiasing by subtracting or adding $\gamma$ for $z=+1$ and $-1$ respectively. After debiasing, the expected value of both groups would match (and be equal to $0$). This complies with the notion of equal mean fairness proposed by \cite{calders2013controlling}. 

The parameter vector $w^* \in \real^d$ is $s$-sparse, i.e., at most $s$ entries of $w^*$ are non-zero. We receive $n$ i.i.d. samples of $X \in \real^d$ and $y \in \real$ and collect them in $\bX \in \real^{n \times d}$  and $\by \in \real^n$ respectively. Thus, in the finite-sample setting,
\begin{align}
\label{eq:sample generative process}
\begin{split}
\by = \bX w^* + \gamma \bz^* + \be \, ,
\end{split}
\end{align}    
where $\bz^* \in \{ -1, 1 \}^n$ and $\be \in \real^n$ both collect $n$ independent realizations of $z^* \in \{-1, 1\}$ and $e \in \real$. 
Our goal is to recover $w^*$ and $\bz^*$ using the samples $(\bX, \by)$. 

We denote a matrix $A \in \real^{p \times q}$ restricted to the columns and rows in $P \subseteq \seq{p}$ and $Q \subseteq \seq{q}$ respectively as $A_{PQ}$. Similarly, a vector $v \in \real^p$ restricted to entries in $P$ is denoted as $v_P$. We use $\eig_i(A)$ to denote the $i$-th eigenvalue ($1$st being the smallest) of matrix $A$. Similarly, $\eig_{\max}(A)$ denotes the maximum eigenvalue of matrix $A$. We use $\diag(A)$ to denote a vector containing the diagonal element of matrix $A$. By overriding the same notation, we use $\diag(v)$ to denote a diagonal matrix with its diagonal being the entries in vector $v$. We denote the inner product between two matrices $A$ and $B$ by $\inner{A}{B}$, i.e., $\inner{A}{B} = \tr(A^\T B)$, where $\tr$ denotes the trace of a matrix. The notation $A \succeq B$ denotes that $A - B$ is a positive semidefinite matrix. Similarly, $A \succ B$ denotes that $A-B$ is a positive definite matrix. For vectors, $\| v \|_p$ denotes the $\ell_p$-vector norm of vector $v \in \real^d$, i.e., $\| v \|_p = ( \sum_{i=1}^d |v_i|^p)^{\frac{1}{p}}$. If $p = \infty$, then we define $\| v \|_{\infty} = \max_{i=1}^d |v_i|$. For matrices, $\| A \|_p$ denotes the induced $\ell_p$-matrix norm for matrix $A \in \real^{p \times q}$. In particular, $\| A \|_2$ denotes the spectral norm of $A$ and $\| A \|_{\infty} \triangleq \max_{i \in \seq{p}} \sum_{j=1}^q |A_{ij}|$. A function $f(x)$ is of order $\Omega(g(x))$ and denoted by $ f(x) = \Omega(g(x))$, if there exists a constant $C > 0$ such that for big enough $x_0$, $f(x) \geq C g(x), \forall x \geq x_0$. Similarly, a function $f(x)$ is of order $\calO(g(x))$ and denoted by $ f(x) = \calO(g(x))$, if there exists a constant $C > 0$ such that for big enough $x_0$, $f(x) \leq C g(x), \forall x \geq x_0$. For brevity in our notations, we treat any quantity independent of $d, s$ and $n$ as constant. Detailed proofs for lemmas and theorems are available in the supplementary material.

\section{Our New Optimization Problem and Invexity}
\label{subsec:optimization problem and invexity}

In this section, we introduce our novel combinatorial problem and propose an invex relaxation. To the best of our knowledge, this is the first invex relaxation for a combinatorial problem. Without any information about the hidden attribute $\bz^*$ in Equation \eqref{eq:sample generative process}, the following LASSO formulation could be incorrectly and unsuccessfully used to estimate the parameter $w^*$. 
\begin{definition}[Standard LASSO]
	\label{def:standard lasso}
	\begin{align}
	\label{eq:opt prob standard lasso}
	\begin{split}
	\begin{matrix}
	\min_{w} & \frac{1}{n} (\bX w - \by)^\T (\bX w - \by) + \lambda_n \| w \|_1 
	\end{matrix} 
	\end{split}
	\end{align}
\end{definition}

However, without including $\bz^*$, standard LASSO does not provide accurate estimation of $w^*$ in Equation \eqref{eq:sample generative process}. We provide the following novel formulation of LASSO which fits our goals of estimating both $w^*$ and $z^*$:

\begin{definition}[Combinatorial Fair LASSO]
	\label{def:fair lasso}
	\begin{align}
	\label{eq:opt prob 1}
	\begin{split}
	\begin{matrix}
	\min_{w, \bz} & \frac{1}{n} (\bX w + \gamma \bz - \by)^\T (\bX w + \gamma \bz - \by) + \lambda_n \| w \|_1, &
	\text{\rm{such that }}  \bz_i \in \{ -1, 1 \}, \, \forall i \in \seq{n}, 
	\end{matrix} 
	\end{split}
	\end{align}
	where 
	$\lambda_n > 0$ is the regularization level which depends on $n$. 
\end{definition}

In its current form, optimization problem \eqref{eq:opt prob 1} is a non-convex mixed integer quadratic program (MIQP). Solving MIQP is NP-hard (See Appendix \ref{sec:miqpnphard}). Next, we will provide a continuous but still non-convex relaxation of \eqref{eq:opt prob 1}. For ease of notation, we define the following quantities:
\begin{align}
\label{eq:M and Z}
\begin{split}
&l(w) \triangleq \frac{1}{n} (\bX w - \by)^\T (\bX w - \by),\; \bZ \triangleq \begin{bmatrix} 1 & \bz^\T \\ \bz & \bz \bz^\T  \end{bmatrix},\;\bM(w) \triangleq \begin{bmatrix} l(w) & \frac{\gamma}{n} (\bX w - \by)^\T \\ \frac{\gamma}{n} (\bX w - \by) & \frac{\gamma^2}{n} \bI_{n \times n}  \end{bmatrix},
\end{split}
\end{align}
where $\bI$ is an $n \times n$ identity matrix. We provide the following invex relaxation to the optimization problem \eqref{eq:opt prob 1}.
\begin{definition}[Invex Fair LASSO]
	\label{def:invex fair lasso}
	\begin{align}
	\label{eq:opt prob 2}
	\begin{split}
	\begin{matrix}
	\min_{w, \bZ} & \inner{\bM(w)}{\bZ} + \lambda_n \| w \|_1, \quad \text{\rm{such that}} & \diag(\bZ) = \mathbf{1}, \; \bZ \succeq \mathbf{0}_{n+1 \times n+1}
	\end{matrix} \, .
	\end{split}
	\end{align}
\end{definition}

Note that optimization problem \eqref{eq:opt prob 2} is continuous and convex with respect to $w$ and $\bZ$ separately but it is not jointly convex (See Appendix \ref{appndx:opt non-convex} for details). Specifically, for a fixed $w$, the matrix $\bM(w)$ becomes a constant and problem \eqref{eq:opt prob 2} resembles a semidefinite program. For a fixed $\bZ$, problem \eqref{eq:opt prob 2} resembles a standard LASSO. Unfortunately, problem \eqref{eq:opt prob 2} is not jointly convex on $w$ and $\bZ$, and thus, it might still remain difficult to solve. Next, we will provide arguments that despite being non-convex, optimization problem \eqref{eq:opt prob 2} belongs to a particular class of non-convex functions namely ``invex'' functions. We define ``invexity'' of functions, as a generalization of convexity~\cite{hanson1981sufficiency}.

\begin{definition}[Invex function]
	\label{def:invex function}
	Let $\phi(t)$ be a function defined on a set $C$. Let $\eta$ be a vector valued function defined in $C \times C$ such that $\eta(t_1, t_2)^\T \nabla \phi(t_2)$, is well defined $\forall t_1, t_2 \in C$. Then, $\phi(t)$ is a $\eta$-invex function if $\phi(t_1) - \phi(t_2) \geq \eta(t_1, t_2)^\T \nabla \phi(t_2), \, \forall t_1, t_2 \in C$.
\end{definition}

Note that convex functions are $\eta$-invex for $\eta(t_1,t_2) = t_1 - t_2$. \cite{hanson1981sufficiency} showed that if the objective function and constraints are both $\eta$-invex with respect to same $\eta$ defined in $C \times C$, then Karush-Kuhn-Tucker (KKT) conditions are sufficient for optimality, while it is well-known that KKT conditions are necessary. \cite{ben1986invexity} showed a function is invex if and only if each of its stationarity point is a global minimum. 

In the next lemma, we show that the relaxed optimization problem \eqref{eq:opt prob 2} is indeed $\eta$-invex for a particular $\eta$ defined in $C \times C$ and a well defined set $C$. Before that, we will reformulate it into an equivalent optimization problem. Note that in the optimization problem \eqref{eq:opt prob 2}, $\diag(\bZ) = \mathbf{1}$. Thus, $\inner{\mathbf{I}}{\bZ}$ is a constant equal to $n+1$. Using this, we can rewrite the optimization problem as: 
\begin{align}
\label{eq:opt prob 3}
\begin{split}
\begin{matrix}
\min_{w, \bZ} & \inner{\bM(w)}{\bZ} + \lambda_n \| w \|_1 + \inner{\mathbf{I}}{\bZ}, \quad \text{such that} & \diag(\bZ) = \mathbf{1}, \;\bZ \succeq \mathbf{0}_{n+1 \times n+1}
\end{matrix} \, ,
\end{split}
\end{align}

Let $C = \{ (w, \bZ) \mid w \in \real^d, \diag(\bZ) = \mathbf{1}, \bZ \succeq \mathbf{0}_{n+1 \times n+1} \}$. We take $\bM'(w) = \bM(w) + \mathbf{I}$ and the corresponding optimization problem becomes: $\min_{(w, \bZ) \in C }  \inner{\bM'(w)}{\bZ} + \lambda_n \| w \|_1$.
We will show that $\forall (w, \bZ) \in C$, $\inner{\bM'(w)}{\bZ} + \lambda_n \| w \|_1$ is an invex function. Note that by definition of the $\ell_1$-norm, $\| w \|_1 = \sup_{\| a \|_{\infty} = 1} \inner{a}{w}$. Thus, it suffices to show that $\forall a \in \real^d$, $\inner{\bM'(w)}{\bZ}$ and $\inner{a}{w}$ are invex for the same $\eta(w, \bar{w}, \bZ, \bar{\bZ})$. 

\begin{lemma}
	\label{lem:invexity}
	For $(w, \bZ) \in C$, the functions $ f(w, \bZ) = \inner{\bM'(w)}{\bZ}$ and $ g(w, \bZ) = \inner{a}{w}$ are $\eta$-invex for $\eta(w, \bar{w}, \bZ, \bar{\bZ}) \triangleq \begin{bmatrix} w - \bar{w} \\ \bM'(\bar{w})^{-1} \bM'(w) (\bZ - \bar{\bZ}) \end{bmatrix}$, where we abuse the vector/matrix notation for clarity of presentation, and avoid the vectorization of matrices.
\end{lemma}
Now that we have established that optimization problem \eqref{eq:opt prob 2} is invex, we are ready to discuss our main results in the next section.

\section{Our Theoretical Analysis}
\label{sec:main results}

In this section, we show that our Invex Fair Lasso formulation correctly recovers the hidden attributes and the regression parameter vector. More formally, we want to achieve the two goals by solving optimization problem \eqref{eq:opt prob 2} efficiently. First, we want to correctly and uniquely determine the hidden sensitive attribute for each data point, i.e., $\bz^* \in \{-1,1\}^n$. Second, we want to recover regression parameter vector which is close to the true parameter vector $w^* \in \real^d$ in $\ell_2$-norm. Let $\tw$ and $\tZ$ be the solution to optimization problem \eqref{eq:opt prob 2}. Then, we will prove that $\tw$ and $w^*$ have the same support and $\tilde{\bz}$ constructed from $\tZ$ is exactly equal to $\bz^*$. We define $\Delta \triangleq (\tw - w^*)$.

\subsection{KKT conditions}
\label{subsec:kkt condtions}

We start by writing the KKT conditions for optimization problem \eqref{eq:opt prob 2}. Let $\balpha \in \real^{n+1}$ and $\Lambda \succeq \bzero_{n+1 \times n+1}$ be the dual variables for optimization problem \eqref{eq:opt prob 2}. For a fixed $\lambda_n$, the Lagrangian $L(w, \bZ; \balpha, \Lambda)$ can be written as $ L(w, \bZ; \balpha, \Lambda) = \inner{\bM(w)}{\bZ} + \lambda_n \| w \|_1 + \inner{\diag(\balpha)}{\bZ} - \mathbf{1}^\T \balpha - \inner{\Lambda}{\bZ}$. 
Using this Lagrangian, the KKT conditions at the optimum can be written as: 

\begin{enumerate}[noitemsep]
	\item Stationarity conditions:
	\begin{align}
	\label{eq:stationarity w}
	\begin{split}
	\diff{\inner{\bM(w)}{\bZ}}{w} + \lambda_n \bg= \mathbf{0}_{d\times 1}, 
	\end{split}
	\end{align}
	where $\bg$ is an element of the subgradient set of $\|w\|_1$, i.e., $\bg \in  \diff{\|w\|_1}{w}$ and $\| \bg \|_{\infty} \leq 1$.
	\begin{align}
	\label{eq:stationarity Z}
	\begin{split}
	\bM(w) + \diag(\balpha) - \Lambda = \mathbf{0}_{n+1 \times n+1}
	\end{split}
	\end{align}
	\item Complementary Slackness condition:
	\begin{align}
	\label{eq:complimentarity}
	\begin{split}
	\inner{\Lambda}{\bZ} = 0 
	\end{split}
	\end{align}
	\item Dual Feasibility condition:
	\begin{align}
	\label{eq:dual feasibility}
	\begin{split}
	\Lambda \succeq \mathbf{0}_{n+1 \times n+1}
	\end{split}
	\end{align}
	\item Primal Feasibility conditions:
	\begin{align}
	\label{eq:primal feasibility}
	\begin{split}
	w \in \real^d, \; \diag(\bZ) = \mathbf{1}, \; \bZ \succeq \mathbf{0}_{n+1 \times n+1}
	\end{split}
	\end{align}
\end{enumerate}

Next, we will provide a setting for primal and dual variables which satisfies all the KKT conditions. But before that, we will describe a set of technical assumptions which will help us in our analysis.

\subsection{Assumptions}
\label{subsec:asssumptions}

Let $S$ denote the support of $w^*$, i.e., $S = \{ i\, | \, w^*_i \ne 0, \, i \in \seq{d} \}$. Similarly, we define the complement of support $S$ as $S^c = \{ i\, | \, w^*_i = 0, \, i \in \seq{d} \}$. Let $|S| = s$ and $|S^c| = d - s$. For ease of notation, we define $\bH \triangleq \E(XX^\T)$ and $\bHhat \triangleq \frac{1}{n}\bX^\T \bX$. As the first assumption, we need the minimum eigenvalue of the population covariance matrix of $X$ restricted to rows and columns in $S$ to be greater than zero. Later, we will show that this assumption is needed to uniquely recover $w$ in the optimization problem \eqref{eq:opt prob 2}.    
\begin{assumption}[Positive Definiteness of Hessian]
	\label{assum:postive definite}
	$\bH_{SS} \succ \mathbf{0}_{s \times s}$ or equivalently $\eig_{\min}(\bH_{SS}) = C_{\min} > 0$. 
\end{assumption}

In practice, we only deal with finite samples and not populations. In the next lemma, we will show that with a sufficient number of samples, a condition similar to Assumption \ref{assum:postive definite} holds with high probability in the finite-sample setting.
\begin{lemma}
	\label{lem:sample positive definite}
	If Assumption \ref{assum:postive definite} holds and $n = \Omega(\frac{s + \log d}{C_{\min}^2})$, then $\eig_{\min}(\bHhat_{SS}) \geq \frac{C_{\min}}{2}$ with probability at least $1 - \calO(\frac{1}{d})$.
\end{lemma}

As the second assumption, we will need to ensure that the variates outside the support of $w^*$ do not exert lot of influence on the variates in the support of $w^*$. This sort of technical condition, known as the mutual incoherence condition, has been previously used in many problems related to regularized regression such as compressed sensing~\cite{wainwright2009sharp}, Markov random fields~\cite{ravikumar2010high}, non-parametric regression~\cite{ravikumar2007spam}, diffusion networks~\cite{daneshmand2014estimating}, among others. We formally present this technical condition in what follows. 

\begin{assumption}[Mutual Incoherence]
	\label{assum:mutual incoherence condition}
	$\| \bH_{S^cS} \bH_{SS}^{-1} \|_{\infty} \leq 1 - \alpha$ for some $\alpha \in (0, 1]$.
\end{assumption}
Again, we will show that with a sufficient number of samples, a condition similar to Assumption \ref{assum:mutual incoherence condition} holds in the finite-sample setting with high probability.
\begin{lemma}
	\label{lem:sample mutual incoherence condition}
	If Assumption \ref{assum:mutual incoherence condition} holds and $n = \Omega(\frac{s^3 (\log s + \log d)}{\tau(C_{\min}, \alpha, \sigma, \Sigma)})$, then $\| \bHhat_{S^cS} \bHhat_{SS}^{-1} \|_{\infty} \leq 1 - \frac{\alpha}{2}$ with probability at least $1 - \calO(\frac{1}{d})$ where $\tau(C_{\min}, \alpha, \sigma, \Sigma)$ is a constant independent of $n, d$ and $s$.
\end{lemma}

In Appendix~\ref{appndx:assumptions hold in sample setting}, we experimentally show that Assumption~\ref{assum:postive definite} is easier to hold (i.e., $n \in \Omega(s + \log d)$) than Assumption~\ref{assum:mutual incoherence condition} (i.e., $n \in \Omega(s^3 \log d)$). Eventually, both assumptions hold as the number of samples increases.

\subsection{Construction of Primal and Dual Witnesses}
In this subsection, we will provide a construction of primal and dual variables which satisfies the KKT conditions for optimization problem \eqref{eq:opt prob 2}. To that end, we provide our first main result. 

\begin{theorem}[Primal Dual Witness Construction]
	\label{thm:primal dual witness construction}
	If Assumptions~\ref{assum:postive definite} and \ref{assum:mutual incoherence condition} hold, $\lambda_n \geq \frac{128 \rho k}{\alpha} \frac{\sqrt{\log d}}{n}$ and $n = \Omega( \frac{s^3 \log d}{\tau_0(C_{\min}, \alpha, \sigma, \Sigma, \rho, k, \gamma)} )$, then the following setting of primal and dual variables 
	\begin{align}
	\label{eq:primal dual variable setting}
	\begin{split}
	&\text{Primal Variables:} \quad \tw = (\tw_S, \bzero_{d-s \times 1}) \\
	& \text{where, } \tw_S = \arg\min_{w_S}  \frac{1}{n} (\bX_{.S} w_S + \gamma \bz^* - \by)^\T  (\bX_{.S} w_S + \gamma \bz^* - \by) + \lambda_n \| w_S \|_1 \\
	& \bZ = \bZ^*  \triangleq \begin{bmatrix} 1 & {\bz^*}^\T \\ {\bz^*} & {\bz^*} {\bz^*}^\T  \end{bmatrix}\\
	&\text{Dual Variables:} \quad \balpha = - \diag(M(\tw) \bZ^*), \quad \Lambda = M(\tw) - \diag(M(\tw)\bZ^*)
	\end{split}
	\end{align}
	satisfies all the KKT conditions for optimization problem \eqref{eq:opt prob 2} with probability at least $1 - \calO(\frac{1}{n})$, where $\tau_0(C_{\min}, \alpha, \sigma, \Sigma, \rho, k, \gamma)$ is a constant independent of $s, d$ and $n$  and thus, the primal variables are a globally optimal solution for \eqref{eq:opt prob 2}. Furthermore, the above solution is also unique.
\end{theorem}
\paragraph{Proof Sketch.} The main idea behind our proofs is to verify that the setting of primal and dual variables in Theorem~\ref{thm:primal dual witness construction} satisfies all the KKT conditions described in subsection \ref{subsec:kkt condtions}. We do this by proving multiple lemmas in subsequent subsections. The outline of the proof is as follows:
\begin{itemize}[noitemsep]
	\item It can be trivially verified that the primal feasibility condition \eqref{eq:primal feasibility} holds. Similarly, the second stationarity condition \eqref{eq:stationarity Z} holds by construction of $\Lambda$.
	\item In subsection~\ref{subsec:verifying the stationarity condition w}, we use  Lemmas~\ref{lem:w_s primal dual} and \ref{lem:bound X_se and X_s^ce} to verify that the stationarity condition~\eqref{eq:stationarity w} holds.
	\item In subsection~\ref{subsec:Verifying Complementary Slackness Condition}, we use Lemma~\ref{lem:eigenvalue of Lambda} to verify the complementary slackness condition \eqref{eq:complimentarity}.
	\item In subsection~\ref{subsec:verifying dual feasibility}, we show that the dual feasibility condition~\eqref{eq:dual feasibility} is satisfied using results from Lemmas~\ref{lem:second eigenvalue of Lambda}, \ref{lem:all positive eigenvalues}, \ref{lem:bound on Delta} and \ref{lem:bound on l2 Xse}.
	\item Finally, in subsection~\ref{subsec:uniqueness of the solution}, we show that our proposed solution is also unique.
\end{itemize}

\subsection{Verifying the Stationarity Condition~\eqref{eq:stationarity w}}
\label{subsec:verifying the stationarity condition w}
In this subsection, we will show that the setting of $\tw$ and $\bZ^*$ satisfies the first stationarity condition \eqref{eq:stationarity w} by proving the following lemma.
\begin{lemma}
	\label{lem:w_s primal dual}
	If Assumptions~\ref{assum:postive definite} and \ref{assum:mutual incoherence condition} hold, $\lambda_n \geq \frac{128 \rho k}{\alpha} \frac{\sqrt{\log d}}{n}$ and $n = \Omega( \frac{s^3 \log d}{\tau_1(C_{\min}, \alpha, \sigma, \Sigma, \rho)} )$, then the setting of $w$ and $\bZ$ from equation \eqref{eq:primal dual variable setting} satisfies the stationarity condition \eqref{eq:stationarity w} with probability at least $1 - \calO(\frac{1}{d})$, where $\tau_1(C_{\min}, \alpha, \sigma, \Sigma, \rho)$ is a constant independent of $d, s$ or $n$.  
\end{lemma}

\subsection{Verifying the Complementary Slackness  \eqref{eq:complimentarity}}
\label{subsec:Verifying Complementary Slackness Condition}
Next, we will show that the setting of $\Lambda$ and $\bZ$ in \eqref{eq:primal dual variable setting} satisfies the complementary slackness condition \eqref{eq:complimentarity}. To this end, we will show the following:
\begin{lemma}
	\label{lem:eigenvalue of Lambda}
	Let $\Lambda$ be defined as in equation \eqref{eq:primal dual variable setting}, then $\zeta^* \triangleq \begin{bmatrix} 1 \\ \bz^* \end{bmatrix}$ is an eigenvector of $\Lambda$ corresponding to the eigenvalue $0$. Furthermore, $\inner{\Lambda}{\bZ^*} = 0$.
\end{lemma}
\begin{proof}
	\label{proof:eigenvalue of Lambda}
	We will show that $\Lambda \zeta^* = \bzero_{n+1 \times 1}$. Note that,
	\begin{align*}
	\begin{split}
	\Lambda = M(w) - \diag(M(w) \bZ^*) = \begin{bmatrix} - \frac{\gamma}{n} (Xw - y)^T  \bz^* & \frac{\gamma}{n} (Xw - y)^T \\ \frac{\gamma}{n} (Xw - y) & -\diag(\frac{\gamma}{n} (Xw - y)  {\bz^*}^T) \end{bmatrix}
	\end{split}
	\end{align*}
	Multiplying the above matrix with $\zeta^*$ gives us $\bzero_{n+1 \times 1}$. Now $\inner{\Lambda}{\bZ^*} = \tr(\Lambda^\T \bZ^*) = \tr(\Lambda \zeta^* {\zeta^*}^\T) = 0$ as $\Lambda \zeta^* = \bzero_{n+1 \times 1} $. 
\end{proof}

\subsection{Verifying the Dual Feasibility \eqref{eq:dual feasibility}}
\label{subsec:verifying dual feasibility}

We have already shown that $\Lambda$ has $0$ as one of its eigenvalues. To verify that it satisfies the dual feasibility condition \eqref{eq:dual feasibility}, we show that second minimum eigenvalue of $\Lambda$ is greater than zero with high probability. At this point, it might not be clear why strict positivity is necessary, but this will be argued later in subsection~\ref{subsec:uniqueness of the solution}. Now, note that: 
\begin{align}
\label{eq:prob second eig value}
\begin{split}
&\prob(\eig_2(\Lambda) > 0) \geq \prob(\eig_2(\Lambda) > 0,\, \| \Delta \|_2 \leq h(n)) \geq \prob(\eig_2(\Lambda) > 0 | \| \Delta \|_2 \leq h(n)) \prob( \| \Delta \|_2 \leq h(n) )
\end{split}
\end{align}
where $h(n)$ is a function of $n$. We bound $\prob(\eig_2(\Lambda) > 0)$ in two parts. First, we bound $\prob(\eig_2(\Lambda) > 0)$ given that $\| \Delta \|_2 \leq h(n)$ and then we bound the probability of $\| \Delta \|_2 \leq h(n)$.
\begin{lemma}
	\label{lem:second eigenvalue of Lambda}
	Given that $\| \Delta \|_2 \leq h(n)$, the second minimum eigenvalue of $\Lambda$ as defined in equation \eqref{eq:primal dual variable setting} is strictly greater than $0$ with probability at least $1 -  \exp(-\frac{\gamma^2}{8 (\rho^2 h(n)^2 + \sigma_e^2)} + \log n) $.
\end{lemma} 
\begin{proof}
	\label{proof:second eigenvalue of Lambda}
	As the first step, we invoke Haynesworth's inertia additivity formula~\cite{haynsworth1968determination} to prove our claim. Let $R$ be a block matrix of the form $ R = \begin{bmatrix} A & B \\ B^\T & C \end{bmatrix}$, then inertia of matrix $R$, denoted by $\inertia(R)$, is defined as the tuple $(\pi(R), \nu(R), \delta(R) )$ where $\pi(R)$ is the number of positive eigenvalues, $\nu(R)$ is the number of negative eigenvalues and $\delta(R)$ is the number of zero eigenvalues of matrix $R$. Haynesworth's inertia additivity formula is given as:
	\begin{align}
	\label{eq:Haynesworth inertia additivity formula}
	\begin{split}
	\inertia(R) = \inertia(C) + \inertia(A - B^\T C^{-1} B)
	\end{split}
	\end{align}
	Note that,
	\begin{align*}
	\begin{split}
	& \Lambda = \bM(w) - \diag(\bM(w) \bZ^*) =\begin{bmatrix} - \frac{\gamma}{n} (\bX w - \by)^\T  \bz^* & \frac{\gamma}{n} (\bX w - \by)^\T \\ \frac{\gamma}{n} (\bX w - \by) & -\diag(\frac{\gamma}{n} (\bX w - \by)  {\bz^*}^\T) \end{bmatrix}
	\end{split}
	\end{align*}
	Then the following holds true by applying equation \eqref{eq:Haynesworth inertia additivity formula}:
	\begin{align*}
	\begin{split}
	\inertia(\Lambda ) =& \inertia( -\diag(\frac{\gamma}{n} (\bX w - \by)  {\bz^*}^\T)  ) + \inertia(  - \frac{\gamma}{n} (\bX w - \by)^\T  \bz^* - \frac{\gamma}{n} (\bX w - \by)^\T (-\diag(\frac{\gamma}{n} (\bX w - \by)  {\bz^*}^\T)^{-1} \\
	&\frac{\gamma}{n} (\bX w - \by)  )
	\end{split}
	\end{align*}
	Notice that the term $- \frac{\gamma}{n} (\bX w - \by)^\T  \bz^* - \frac{\gamma}{n} (\bX w - \by)^\T (-\diag(\frac{\gamma}{n} (\bX w - \by)  {\bz^*}^\T)^{-1} \frac{\gamma}{n} (\bX w - \by))$ evaluates to $0$. Thus, it has $0$ positive eigenvalue, $0$ negative eigenvalue and $1$ zero eigenvalue. We have also shown in Lemma \ref{lem:eigenvalue of Lambda} that  $\Lambda$ has at least $1$ zero eigenvalue. It follows that
	\begin{align}
	\label{eq:inertia}
	\begin{split}
	&\pi(\Lambda) = \pi(-\diag(\frac{\gamma}{n} (\bX w - \by) {\bz^*}^\T) ), \quad \nu(\Lambda) = \nu(-\diag(\frac{\gamma}{n} (\bX w - \by)  {\bz^*}^\T) ) \\
	&\delta(\Lambda) = \delta(-\diag(\frac{\gamma}{n} (\bX w - \by)  {\bz^*}^\T) ) + 1
	\end{split}
	\end{align}
	
	Next, we will show that $-\diag(\frac{\gamma}{n} (\bX w - \by)  {\bz^*}^\T)$ has all of its eigenvalues being positive.
	\begin{lemma}
		\label{lem:all positive eigenvalues}
		For a given $\| \Delta \|_2 \leq h(n)$, all eigenvalues of $-\diag(\frac{\gamma}{n} (\bX w - \by)  {\bz^*}^\T)$ are strictly greater than $0$ with probability at least $1 - \exp(-\frac{\gamma^2}{8 (\rho^2 h(n)^2 + \sigma_e^2)} + \log n) $. 
	\end{lemma}
	\begin{proof}
		Using equation \eqref{eq:generative process}, we can expand the term $-\diag(\frac{\gamma}{n} (\bX w - \by)  {\bz^*}^\T)$ as $-\diag(\frac{\gamma}{n} (\bX (w - w^*) - \gamma \bz^* - \be)  {\bz^*}^\T)$. Since eigenvalues of a diagonal matrix are its diagonal elements, we focus on the $i$-th diagonal element of $-\diag(\frac{\gamma}{n} (\bX \Delta - \gamma \bz^* - \be)  {\bz^*}^\T)$ which is $\frac{\gamma^2}{n} - \frac{\gamma}{n} z_i^* (\bX_{i\cdot}^\T \Delta + \be_i)$. Note that $(\bX_{i\cdot}^\T \Delta + \be_i)$ is a sub-Gaussian random variable with parameter $\rho^2 \| \Delta \|_2^2 + \sigma_e^2$. Using the tail inequality for sub-Gaussian random variables, for some $t > 0$, we can write:
		\begin{align*}
		\begin{split}
		\prob( (\bX_{i\cdot} \Delta + \be_i) \geq t) \leq \exp(-\frac{t^2}{2 (\rho^2 \| \Delta \|_2^2 + \sigma_e^2)})
		\end{split}
		\end{align*}
		We take union bound across all the diagonal elements and replace $t = \frac{\gamma}{2}$ and $\| \Delta \|_2 \leq h(n)$ to complete the proof, i.e.,
		\begin{align}
		\label{eq:second eig value eq}
		\begin{split}
		\prob(\exists i \in \seq{n} \, , (\bX_{i\cdot} \Delta + \be_i) \geq t) \leq n \exp(-\frac{\gamma^2}{8 (\rho^2 h(n)^2 + \sigma_e^2)}) \, .
		\end{split}
		\end{align} 
	\end{proof}
	The result of Lemma~\ref{lem:second eigenvalue of Lambda} follows directly from Lemma~\ref{lem:all positive eigenvalues} and equation~\eqref{eq:inertia}.
\end{proof}

Now, we are ready to bound $\| \Delta \|_2$. Due to our primal dual construction, $\| \Delta \|_2$ is simply equal to $\| \Delta_S \|_2$. We provide a bound on $\Delta_S$ in the following lemma:
\begin{lemma}
	\label{lem:bound on Delta}
	If Assumptions~\ref{assum:postive definite} and \ref{assum:mutual incoherence condition} hold, $\lambda_n \geq \frac{128 \rho k \sqrt{\log d}}{\alpha n}$ and $n = \Omega( \frac{s^3 \log d}{\tau_2(C_{\min}, \rho, k)})$, then $\| \Delta_S \|_2 \leq \frac{2\lambda_n \sqrt{s}}{C_{\min}}$ with probability at least $1 - \calO(\frac{1}{d})$ where $\tau_2(C_{\min}, \rho, k)$ is a constant independent of $s, d$ or $n$.
\end{lemma}

By taking $h(n) = \frac{2 \lambda_n \sqrt{s} }{C_{\min}} $ in \eqref{eq:prob second eig value}, we get the following: $\prob(\eig_2(\Lambda) > 0) \geq 1 - \calO(\frac{1}{n})$, 
as long as $n = \Omega( \frac{s^3 \log d}{\tau_3(C_{\min}, \rho, k, \alpha, \gamma)})$, where $\tau_3(C_{\min}, \rho, k, \alpha, \gamma)$ is a constant independent of $s, d$ and $n$. The above results combined with the property that optimization problem \eqref{eq:opt prob 2} is invex ensure that the setting of primal and dual variables in Theorem~\ref{thm:primal dual witness construction} is indeed the globally optimal solution to the problem \eqref{eq:opt prob 2}. It remains to show that this solution is also unique.

\subsection{Uniqueness of the Solution}
\label{subsec:uniqueness of the solution}

First, we prove that $\bZ^*$ is a unique solution. Suppose there is another solution $\bar{\bZ}$ which satisfies all KKT conditions and is optimal. Then, $\bar{\bZ} \succeq \bzero_{n+1 \times n+1}$ and $\inner{\Lambda}{\bar{\bZ}} = 0$. Since, $\Lambda \succeq \bzero_{n+1 \times n+1}$ and $\eig_2(\Lambda) > 0$, $\zeta^*$ spans all of its null space. This enforces that $\bar{\bZ}$ is a multiple of $\bZ^*$. But primal feasibility dictates that $\diag(\bar{\bZ}) = \mathbf{1}$. It follows that $\bar{\bZ} = \bZ^*$. To show that $\tw$ is unique, it suffices to show that $\tw_S$ is unique. After substituting $\bZ = \bZ^*$, we observe that the Hessian of optimization problem \eqref{eq:opt prob 2} with respect to $w$ and restricted to rows and columns in $S$, i.e., $\bHhat_{SS}$ is positive definite. This ensures that $\tw$ is a unique solution.

The setting of primal and dual variables in Theorem \ref{thm:primal dual witness construction} not only solves the optimization problem \eqref{eq:opt prob 2} but also gives rise to the following results:
\begin{corollary}
	\label{cor:primal dual vairables}
	If Assumptions~\ref{assum:postive definite} and \ref{assum:mutual incoherence condition} hold, $\lambda_n \geq \frac{128 \rho k}{\alpha} \frac{\sqrt{\log d}}{n}$ and $n = \Omega( \frac{s^3 \log d}{\tau_1(C_{\min}, \alpha, \sigma, \Sigma, \rho)} )$, then the following statements are true with probability at least $1 - \calO(\frac{1}{n})$:
	\begin{enumerate}[noitemsep]
		\item The solution $\bZ$ correctly recovers hidden attribute for each sample, i.e., $\bZ = \bZ^* = \zeta^* {\zeta^*}^\T$.
		\item The support of recovered regression parameter $\tw$ matches exactly with the support of $w^*$.
		\item If $\min_{i \in S} |w_i^*| \geq  \frac{4 \lambda_n \sqrt{s} }{C_{\min}} $ then for all $i \in \seq{d}$, $\tw_i$ and $w_i^*$ match up to their sign. 
	\end{enumerate}
\end{corollary}

\section{Experimental Validation}
\label{sec:experimental results}

\paragraph{Synthetic Experiments.} We validate our theoretical result in Theorem \ref{thm:primal dual witness construction} and Corollary~\ref{cor:primal dual vairables} by conducting experiments on synthetic data. We show that for a fixed $s$, we need $n = 10^\beta \log d$ samples for recovering the exact support of $w^*$ and exact hidden attributes $\bZ^*$, where $\beta \equiv \beta(s, C_{\min}, \alpha, \sigma, \Sigma, \rho, \gamma, k)$ is a control parameter which is independent of $d$. We draw $\bX \in \real^{n \times d }$ and $\be \in \real^n$ from Gaussian distributions. We randomly generate $w^* \in \real^d$ with $s = 10$ non-zero entries. Regarding the hidden attribute $\bz^* \in \{-1,1\}^n$, we set $\frac{n}{2}$ entries as $+1$ and the rest as $-1$. The response $\by \in \real^n$ is generated according to \eqref{eq:generative process}. According to Theorem~\ref{thm:primal dual witness construction}, the regularizer $\lambda_n$ is chosen to be equal to $\frac{128 \rho k}{\alpha} \frac{\sqrt{\log d}}{n}$. We solve optimization problem \eqref{eq:opt prob 2} by using an alternate optimization algorithm that converges to the optimal solution (See Appendix \ref{appndx:alternate optimization} for details). Figure~\ref{fig:recnumsample} shows that our method recovers the true support as we increase the number of samples. Similarly, Figure~\ref{fig:labelnumsample} shows that as the number of samples increase, our recovered hidden attributes are 100\% correct. Curves line up perfectly in Figure~\ref{fig:recnumsamplecp} and \ref{fig:labelnumsamplecp} when plotting with respect to the control parameter $\beta = \log \frac{n}{\log d}$. This validates our theoretical results (Details in Appendix \ref{appndx:experimental results}). 

\paragraph{Real World Experiments.} We show applicability of our method by identify groups with bias in the Communities and Crime data set~\cite{redmond2002data} and the Student Performance data set~\cite{cortez2008using}. In both cases, our method is able to recover groups with bias (Details in Appendix \ref{appndx:real world experiments}).

\begin{figure*}[!ht]
	\centering
	\begin{subfigure}{.25\textwidth}
		\centering
		\includegraphics[width=\linewidth]{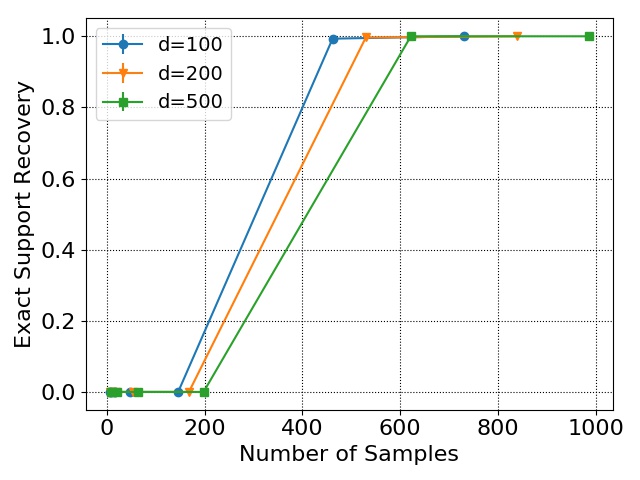}
		\caption{Recovery of $S$ vs $n$}
		\label{fig:recnumsample}
	\end{subfigure}%
	\begin{subfigure}{.25\textwidth}
		\centering
		\includegraphics[width=\linewidth]{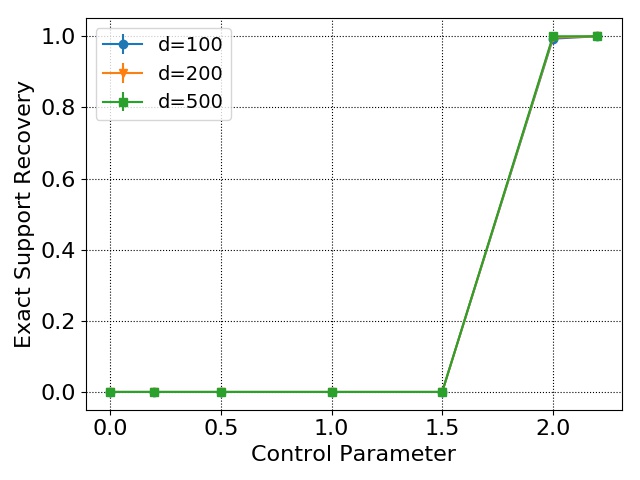}	
		\caption{Recovery of $S$ vs $\beta$}
		\label{fig:recnumsamplecp}
	\end{subfigure}%
	\begin{subfigure}{.25\textwidth}
		\centering
		\includegraphics[width=\linewidth]{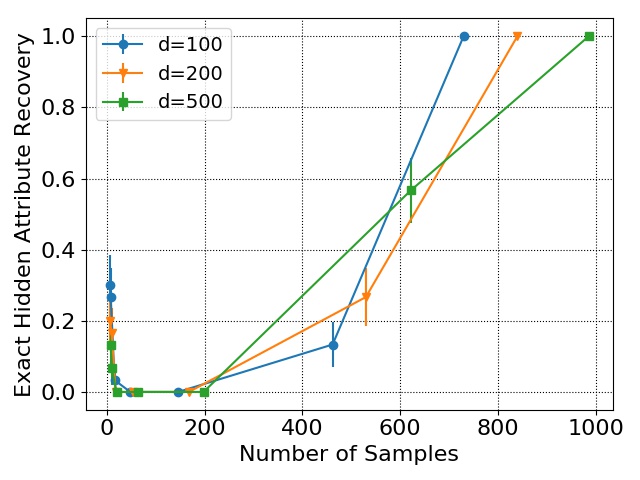}
		\caption{Recovery of $\bZ^*$ vs $n$}
		\label{fig:labelnumsample}
	\end{subfigure}%
	\begin{subfigure}{.25\textwidth}
		\centering
		\includegraphics[width=\linewidth]{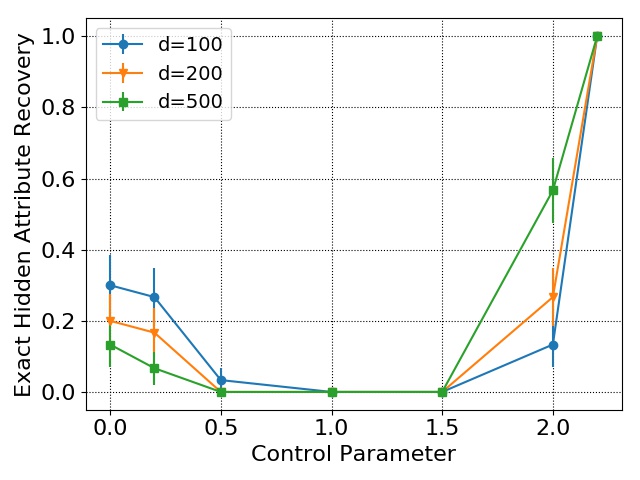}	
		\caption{Recovery of $\bZ^*$ vs $\beta$}
		\label{fig:labelnumsamplecp}
	\end{subfigure}%
	\caption{Left two: Exact support recovery of $w^*$ across $30$ runs. Right two: Exact hidden attribute recovery of $\bZ^*$ across $30$ runs.}
	\label{fig:recovery}
\end{figure*}


\clearpage

\appendix

\noindent\rule{\textwidth}{3pt}
\begin{center}
	{\Large \textbf{Supplementary Material: Fair Sparse Regression with Clustering: An Invex Relaxation for a Combinatorial Problem}}
\end{center}
\noindent\rule{\textwidth}{1pt}

\section{Proof of Lemma \ref{lem:invexity}}
\label{sec:proof of lem:invexity}

\paragraph{Lemma \ref{lem:invexity}}
\emph{For $(w, \bZ) \in C$, the functions $ f(w, \bZ) = \inner{\bM'(w)}{\bZ}$ and $ g(w, \bZ) = \inner{a}{w}$ are $\eta$-invex for $\eta(w, \bar{w}, \bZ, \bar{\bZ}) \triangleq \begin{bmatrix} w - \bar{w} \\ \bM'(\bar{w})^{-1} \bM'(w) (\bZ - \bar{\bZ}) \end{bmatrix}$, where we abuse the vector/matrix notation for clarity of presentation, and avoid the vectorization of matrices.}
\begin{proof}
	We need to prove the following two inequalities. 
	\begin{align}
	\label{eq:invexity 1}
	\begin{split}
	& f(w, \bZ) - f(\bar{w}, \bar{\bZ}) - \inner{ \nabla_{\bar{w}, \bar{\bZ}} f(w, \bZ)}{\eta(w, \bar{w}, \bZ, \bar{\bZ})} \geq 0 \, ,
	\end{split}
	\end{align}
	\begin{align}
	\label{eq:invexity 2}
	\begin{split}
	&g(w, \bZ) - g(\bar{w}, \bar{\bZ}) - \inner{ \nabla_{\bar{w}, \bar{\bZ}} g(w, \bZ)}{\eta(w, \bar{w}, \bZ, \bar{\bZ})} \geq 0 \, .
	\end{split}
	\end{align}
	First, we observe that function $g(w, \bZ)$ only depends on $w$ and moreover, $\forall a \in \real^d$, $g(w, \bZ)$ is convex in $w$. Thus, the inequality~\eqref{eq:invexity 2} holds trivially. 
	Note that $f(w, \bZ) = \inner{\bM'(w)}{\bZ} = \sum_{ij} \bM_{ij}'(w) \bZ_{ij}$. Then, 
	\begin{align*}
	\begin{split}
	\diff{f(\bar{w}, \bar{\bZ})}{w} = \sum_{ij} \bar{\bZ}_{ij} \diff{\bM_{ij}'(\bar{w})}{w}, \;\;\; \diff{f(\bar{w}, \bar{\bZ})}{\bZ} = \bM'(\bar{w})
	\end{split}
	\end{align*}
	We further note that the diagonal elements of $\bM'(w)$ are convex with respect to $w$ and the off diagonal elements are linear. Therefore, we can write the following:
	\begin{align*}
	\begin{split}
	\bM_{ii}'(w) - \bM_{ii}'(\bar{w}) &\geq \inner{\diff{\bM_{ii}'(\bar{w})}{w}}{w - \bar{w}}, \forall i \in \seq{n+1} \\
	\bM_{ij}'(w) - \bM_{ij}'(\bar{w}) &= \inner{\diff{\bM_{ij}'(\bar{w})}{w}}{w - \bar{w}},  \forall i, j \in \seq{n+1}, i \ne j
	\end{split}
	\end{align*}
	Since $\bar{\bZ}_{ii} \geq 0$, it follows that
	\begin{align*}
	\begin{split}
	\bar{\bZ}_{ij} \inner{\diff{\bM_{ij}'(\bar{w})}{w}}{w - \bar{w}} \leq \bar{\bZ}_{ij} (\bM_{ij}'(w) - \bM_{ij}'(\bar{w})) \, .
	\end{split}
	\end{align*}
	Now, we prove that $f(w, \bZ)$ is indeed $\eta$-invex, that is 
	\begin{align*}
	\begin{split}
	& f(w, \bZ) - f(\bar{w}, \bar{\bZ}) - \inner{ \nabla_{\bar{w}, \bar{\bZ}} f(w, \bZ)}{\eta(w, \bar{w}, \bZ, \bar{\bZ})} \\
	&= \inner{\bM'(w)}{\bZ} - \inner{\bM'(\bar{w})}{\bar{\bZ}} - \inner{\sum_{ij} \bar{\bZ}_{ij} \diff{\bM_{ij}'(\bar{w})}{w}}{w - \bar{w}} - \inner{\bM'(\bar{w})}{\bM'(\bar{w})^{-1} \bM'(w) (\bZ - \bar{\bZ})} \\
	&\geq  \inner{\bM'(w)}{\bZ} - \inner{\bM'(\bar{w})}{\bar{\bZ}} -  \sum_{ij} \bar{\bZ}_{ij}  (\bM_{ij}'(w) - \bM_{ij}'(\bar{w})) - \inner{\bM'(w)}{\bZ}  + \inner{\bM'(w)}{\bar{\bZ}}\\
	&= 0
	\end{split}
	\end{align*}
	This proves that $f(w, \bZ)$ is $\eta$-invex in $(w, \bZ) \in C$.  
\end{proof}

\section{Mixed Integer Quadratic Program (MIQP) \eqref{eq:opt prob 1} is NP-Hard}
\label{sec:miqpnphard}
In this section, we will show that the MIQP presented in \eqref{eq:opt prob 1} is at least as hard to solve as a $0-1$ Quadratic Program. It should be noted that MIQP \eqref{eq:opt prob 1} is stated for a fixed $\bX$. However, since the entries in $\bX$ are drawn from a sub-Gaussian distribution, matrix $\bX$ can potentially realize any real matrix in $\real^{n \times d}$.      
\begin{lemma}
	\label{lem:miqpnphard}
	The Mixed Integer Quadratic Program (MIQP) \eqref{eq:opt prob 1} is NP-hard. 
\end{lemma}
\begin{proof}
	We will consider the case when $\lambda_n = 0$. Other cases will be at least as difficult as this case. First, we write optimization problem \eqref{eq:opt prob 1} in the following form:
	\begin{align}
	\begin{split}
	&\min_{w\in \real^d, \bz \in \{-1, 1\}^n} \frac{1}{2} w^\T (\frac{2}{n} \bX^\T \bX)w + \bz^\T \frac{2}{n} \gamma \bX w + \frac{1}{2} \bz^\T \frac{2}{n} \gamma^2 \bI \bz - \frac{2}{n} \by^\T \bX w - \frac{2}{n} \gamma \by^\T \bz \\
	&=\min_{\bz \in \{-1, 1\}^n}  \left( \frac{1}{2} \bz^\T \frac{2}{n} \gamma^2 \bI \bz - \frac{2}{n} \gamma \by^\T \bz  + \left( \min_{w\in \real^d} \frac{1}{2} w^\T (\frac{2}{n} \bX^\T \bX)w + ( \bz^\T \frac{2}{n} \gamma \bX - \frac{2}{n} \by^\T \bX  ) w  \right) \right)
	\end{split}
	\end{align}
	We observe that $w = (\bX^\T \bX)^{\dagger} \bX^\T (-\gamma \bz + \by)$ solves the nested optimization problem, where $(\cdot)^\dagger$ denotes the pseudo-inverse. Thus, substituting the optimal value of $w$, we get the following optimization problem: 
	\begin{align}
	\begin{split}
	\min_{\bz \in \{-1, 1\}^n} \frac{\gamma^2}{n} \bz^\T (\bI - \bX (\bX^\T \bX)^\dagger \bX^\T) \bz - \frac{2 \gamma}{n} \by^\T (\bI - \bX (\bX^\T \bX)^\dagger \bX^\T) \bz
	\end{split}
	\end{align}
	Observe that $I - \bX (\bX^\T \bX)^\dagger \bX^\T$ can potentially be any fixed real matrix in $\real^{n \times n}$. By simply substituting $\bz' = \frac{\bz + 1}{2}$, we get a $0-1$ Quadratic Program which is known to be NP-Hard~\cite{billionnet2007using}. 
\end{proof}


\section{Proof of Lemma \ref{lem:sample positive definite}}
\label{appndx:lem:sample positive definite}
\paragraph{Lemma \ref{lem:sample positive definite}}
\emph{If Assumption \ref{assum:postive definite} holds and $n = \Omega(\frac{s + \log d}{C_{\min}^2})$, then $\eig_{\min}(\bHhat_{SS}) \geq \frac{C_{\min}}{2}$ with probability at least $1 - \calO(\frac{1}{d})$.}
\begin{proof}
	\label{proof:sample positive definite}
	By the Courant-Fischer variational representation~\cite{horn2012matrix}:
	\begin{align}
	\begin{split}
	\eig_{\min}(\E(XX^\T)_{SS}) = \min_{\|y \|_2 = 1} y^\T \E(XX^\T)_{SS} y &= \min_{\|y \|_2 = 1} y^\T (\E(XX^\T)_{SS} - \frac{1}{n} \bX^\T_S \bX_S + \frac{1}{n} \bX^\T_S \bX_S ) y \\
	&\leq y^\T (\E(XX^\T)_{SS} - \frac{1}{n} \bX^\T_S \bX_S + \frac{1}{n} \bX^\T_S \bX_S ) y \\
	&= y^\T (\E(XX^\T)_{SS} - \frac{1}{n} \bX^\T_S \bX_S) y + y^\T \frac{1}{n} \bX^\T_S \bX_S  y
	\end{split}
	\end{align}
	It follows that
	\begin{align}
	\begin{split}
	\eig_{\min}(\frac{1}{n} \bX^\T_S \bX_S) \geq C_{\min} - \| \E(XX^\T)_{SS} - \frac{1}{n} \bX^\T_S \bX_S \|_2
	\end{split}
	\end{align}
	The term $\| \E(XX^\T)_{SS} - \frac{1}{n} \bX^\T_S \bX_S \|_2$ can be bounded using Proposition 2.1 in \cite{vershynin2012close} for sub-Gaussian random variables. In particular,
	\begin{align}
	\begin{split}
	\prob(\| \E(XX^\T)_{SS} - \frac{1}{n} \bX^\T_S \bX_S \|_2 \geq \epsilon) \leq 2  \exp(-c\epsilon^2n + s) 
	\end{split}
	\end{align}
	for some constant $c > 0$. Taking $\epsilon = \frac{C_{\min}}{2}$, we show that $\eig_{\min}(\frac{1}{n}\bX^\T_S\bX_S) \geq \frac{C_{\min}}{2}$ with probability at least $1 - 2  \exp(- \frac{cC_{\min}^2 n}{4} + s)$.
\end{proof}

\section{Proof of Lemma \ref{lem:sample mutual incoherence condition}}
\label{appndx:lem:sample mutual incoherence condition}
\paragraph{Lemma \ref{lem:sample mutual incoherence condition}}
\emph{If Assumption \ref{assum:mutual incoherence condition} holds and $n = \Omega(\frac{s^3 (\log s + \log d)}{\tau(C_{\min}, \alpha, \sigma, \Sigma)})$, then $\| \bHhat_{S^cS} \bHhat_{SS}^{-1} \|_{\infty} \leq 1 - \frac{\alpha}{2}$ with probability at least $1 - \calO(\frac{1}{d})$ where $\tau(C_{\min}, \alpha, \sigma, \Sigma)$ is a constant independent of $n, d$ and $s$.}
\begin{proof}
	\label{proof:sample mutual incoherence condition}
	Before we prove the result of Lemma \ref{lem:sample mutual incoherence condition}, we will prove a helper lemma. 
	\begin{lemma}
		\label{lem:helper mutual incoherence}
		If Assumption \ref{assum:mutual incoherence condition} holds then for some $\delta > 0$, the following inequalities hold:
		\begin{align}
		\label{eq:helper mutual incoherence}
		\begin{split}
		&\prob( \| \bHhat_{S^cS} - \bH_{S^cS} \|_{\infty} \geq \delta ) \leq 4 (d - s) s \exp( - \frac{n \delta^2}{128 s^2 (1+4\sigma^2) \max_{l} \Sigma_{ll}^2}) \\
		&\prob( \| \bHhat_{SS} - \bH_{SS} \|_{\infty} \geq \delta ) \leq 4 s^2 \exp( - \frac{n \delta^2}{128 s^2 (1+4\sigma^2) \max_{l} \Sigma_{ll}^2}) \\
		&\prob( \| (\bHhat_{SS})^{-1} - (\bH_{SS})^{-1} \|_{\infty} \geq \delta ) \leq 2\exp(- \frac{c\delta^2 C_{\min}^4n}{4 s} + s)+ 2 \exp( - \frac{cC_{\min}^2n}{4} + s)
		\end{split}
		\end{align}
	\end{lemma}
	\begin{proof}
		\label{proof:helper mutual incoherence}
		Let $A_{ij}$ be $(i, j)$-th entry of $\bHhat_{S^cS} - \bH_{S^cS}$. Clearly, $\E(A_{ij}) = 0$. By using the definition of the $\| \cdot\|_{\infty}$ norm, we can write:
		\begin{align}
		\begin{split}
		\prob(\| \bHhat_{S^cS} - \bH_{S^cS} \|_{\infty} \geq \delta) &= \prob(\max_{i \in S^c} \sum_{j \in S} |A_{ij}| \geq \delta) \\
		& \leq (d - s) \prob(\sum_{j \in S} |A_{ij}| \geq \delta) \\
		&\leq (d - s) s \prob(|A_{ij}| \geq \frac{\delta}{s})
		\end{split}
		\end{align}
		where the second last inequality comes as a result of the union bound across entries in $S^c$ and the last inequality is due to the union bound across entries in $S$. Recall that $X_i, i \in \seq{d}$ are zero mean random variables with covariance $\Sigma$ and each $\frac{X_i}{\sqrt{\Sigma_{ii}}}$ is a sub-Gaussian random variable with parameter $\sigma$. Using the results from Lemma 1 of \cite{ravikumar2011high}, for some $\delta \in (0, s \max_{l} \Sigma_{ll} 8(1 + 4 \sigma^2))$, we can write:
		\begin{align}
		\begin{split}
		\prob(|A_{ij}| \geq \frac{\delta}{s}) \leq 4 \exp( - \frac{n \delta^2}{128 s^2 (1+4\sigma^2) \max_{l} \Sigma_{ll}^2})
		\end{split}
		\end{align} 
		Therefore,
		\begin{align}
		\begin{split}
		&\prob(\| \bHhat_{S^cS} - \bH_{S^cS} \|_{\infty} \geq \delta)  \leq 4 (d - s) s \exp( - \frac{n \delta^2}{128 s^2 (1+4\sigma^2) \max_{l} \Sigma_{ll}^2})
		\end{split}
		\end{align}
		Similarly, we can show that 
		\begin{align}
		\begin{split}
		&\prob(\| \bHhat_{SS} - \bH_{SS} \|_{\infty} \geq \delta)  \leq 4 s^2 \exp( - \frac{n \delta^2}{128 s^2 (1+4\sigma^2) \max_{l} \Sigma_{ll}^2})
		\end{split}
		\end{align}
		Next, we will show that the third inequality in \eqref{eq:helper mutual incoherence} holds. Note that
		\begin{align}
		\begin{split}
		\| (\bHhat_{S^cS})^{-1} - (\bH_{S^cS})^{-1} \|_{\infty} &= \| (\bH_{SS})^{-1} (\bH_{SS} - \bHhat_{SS}) (\bHhat_{SS})^{-1} \|_{\infty} \\
		&\leq \sqrt{s}  \| (\bH_{SS})^{-1} (\bH_{SS} - \bHhat_{SS}) (\bHhat_{SS})^{-1} \|_{2} \\
		&\leq \sqrt{s}  \| (\bH_{SS})^{-1} \|_2 \| (\bH_{SS} - \bHhat_{SS})\|_2 \| (\bHhat_{SS})^{-1} \|_{2} \\
		\end{split}
		\end{align}   
		Note that $\| \bH_{SS} \|_2 \geq C_{\min}$, thus $\| (\bH_{SS})^{-1} \|_2 \leq \frac{1}{C_{\min}}$. Similarly,  $\| \bH_{SS} \|_2 \geq \frac{C_{\min}}{2}$ with probability at least $1 - 2 \exp( - \frac{cC_{\min}^2n}{4} + s)$. We also have $\| (\bH_{SS} - \bHhat_{SS})\|_2 \leq \epsilon$ with probability at least $1 - 2\exp(-c\epsilon^2 n + s)$. Taking $\epsilon = \delta \frac{C_{\min}^2}{2 \sqrt{s}}$, we get 
		\begin{align}
		\begin{split}
		\prob( \| (\bH_{SS} - \bHhat_{SS})\|_2 \geq  \delta \frac{C_{\min}^2}{2 \sqrt{s}} ) \leq 2\exp(- \frac{c\delta^2 C_{\min}^4n}{4 s} + s)
		\end{split}
		\end{align}
		It follows that $\| (\bHhat_{SS})^{-1} - (\bH_{SS})^{-1} \|_{\infty} \leq \delta$ with probability at least $1 - 2\exp(- \frac{c\delta^2 C_{\min}^4n}{4 s} + s) - 2 \exp( - \frac{cC_{\min}^2n}{4} + s)$.
	\end{proof}
	Now we are ready to show that the statement of Lemma \ref{lem:sample mutual incoherence condition} holds using the results from Lemma \ref{lem:helper mutual incoherence}. We will rewrite $\bHhat_{S^cS} (\bHhat_{SS})^{-1}$ as the sum of four different terms:
	\begin{align}
	\begin{split}
	\bHhat_{S^cS} (\bHhat_{SS})^{-1} = T_1 + T_2 + T_3 + T_4,
	\end{split}
	\end{align} 
	where 
	\begin{align}
	\begin{split}
	T_1 &\triangleq \bHhat_{S^cS} ( (\bHhat_{SS})^{-1} - (\bH_{SS})^{-1} ) \\
	T_2 &\triangleq (\bHhat_{S^cS} - \bH_{S^cS}) (\bH_{SS})^{-1} \\
	T_3 &\triangleq (\bHhat_{S^cS} - \bH_{S^cS})((\bHhat_{SS})^{-1} - (\bH_{SS})^{-1}) \\
	T_4 &\triangleq \bH_{S^cS} (\bH_{SS})^{-1}   \, .
	\end{split}
	\end{align}
	Then it follows that  $\| \bHhat_{S^cS} (\bHhat_{SS})^{-1} \|_{\infty} \leq \| T_1 \|_{\infty} + \| T_2 \|_{\infty} + \| T_3 \|_{\infty} + \| T_4 \|_{\infty}$. Now, we will bound each term separately. First, recall that Assumption \ref{assum:mutual incoherence condition} ensures that $\| T_4 \|_{\infty} \leq 1 - \alpha$.
	\paragraph{Controlling $T_1$.} We can rewrite $T_1$ as,
	\begin{align}
	\begin{split}
	T_1 = - \bH_{S^cS} (\bH_{SS})^{-1} (\bHhat_{SS} - \bH_{SS}) (\bHhat_{SS})^{-1}
	\end{split}
	\end{align}
	then,
	\begin{align}
	\begin{split}
	\| T_1 \|_{\infty} &= \| \bH_{S^cS} (\bH_{SS})^{-1} (\bHhat_{SS} - \bH_{SS}) (\bHhat_{SS})^{-1} \|_{\infty} \\
	&\leq \| \bH_{S^cS} (\bH_{SS})^{-1} \|_{\infty} \| (\bHhat_{SS} - \bH_{SS})\|_{\infty} \| (\bHhat_{SS})^{-1} \|_{\infty} \\
	&\leq (1 - \alpha)  \| (\bHhat_{SS} - \bH_{SS})\|_{\infty} \sqrt{s} \| (\bHhat_{SS})^{-1} \|_2 \\
	&\leq (1 - \alpha) \| (\bHhat_{SS} - \bH_{SS})\|_{\infty} \frac{2\sqrt{s}}{C_{\min}} \\
	&\leq \frac{\alpha}{6} 
	\end{split}
	\end{align}
	The last inequality holds with probability at least $1 - 2\exp( - \frac{c C_{\min}^2 n}{4} + s) - 4s^2 \exp( - \frac{n C_{\min}^2 \alpha^2 }{18432(1-\alpha)^2 s^3(1 + 4\sigma^2) \max_{l} \Sigma_{ll}^2} )$ by taking $\delta = \frac{C_{\min} \alpha}{12 (1 - \alpha) \sqrt{s}}$.
	\paragraph{Controlling $T_2$.} Recall that $T_2 = (\bHhat_{S^cS} - \bH_{S^cS}) (\bH_{SS})^{-1}$. Thus,
	\begin{align}
	\begin{split}
	\| T_2 \|_{\infty} &\leq \sqrt{s} \| (\bH_{SS})^{-1} \|_2 \| (\bHhat_{S^cS} - \bH_{S^cS}) \|_{\infty} \\
	&\leq \frac{\sqrt{s}}{C_{\min}}   \| (\bHhat_{S^cS} - \bH_{S^cS}) \|_{\infty} \\
	&\leq \frac{\alpha}{6}
	\end{split}
	\end{align}
	The last inequality holds with probability at least $1 - 4(d-s)s \exp( -\frac{n C_{\min}^2 \alpha^2 }{4608 s^3 (1 + 4\sigma^2) \max_{l} \Sigma_{ll}^2 } )$ by choosing $\delta = \frac{C_{\min} \alpha}{6 \sqrt{s}}$.
	\paragraph{Controlling $T_3$.} Note that,
	\begin{align}
	\begin{split}
	\| T_3 \|_{\infty} &\leq \| (\bHhat_{S^cS} - \bH_{S^cS}) \|_{\infty} \| ((\bHhat_{SS})^{-1} - (\bH_{SS})^{-1}) \|_{\infty} \\
	&\leq \frac{\alpha}{6}
	\end{split}
	\end{align}
	The last inequality holds with probability at least $1 - 4(d-s)s \exp(- \frac{n \alpha}{768s^2 (1 + 4\sigma^2) \max_{l} \Sigma_ll^2 }) - 2\exp(- \frac{c\alpha C_{\min}^4n}{24 s} + s) - 2 \exp( - \frac{cC_{\min}^2n}{4} + s)$ by choosing $\delta = \sqrt{\frac{\alpha}{6}}$ in the first and third inequality of equation \eqref{eq:helper mutual incoherence}. By combining all the above results, we prove Lemma \ref{lem:sample mutual incoherence condition}. 
\end{proof}

\section{Proof of Lemma \ref{lem:w_s primal dual}}
\label{appndx:proof of lem:w_s primal dual}

\paragraph{Lemma \ref{lem:w_s primal dual}}
\emph{If Assumptions~\ref{assum:postive definite} and \ref{assum:mutual incoherence condition} hold, $\lambda_n \geq \frac{128 \rho k}{\alpha} \frac{\sqrt{\log d}}{n}$ and $n = \Omega( \frac{s^3 \log d}{\tau_1(C_{\min}, \alpha, \sigma, \Sigma, \rho)} )$, then the setting of $w$ and $\bZ$ from equation \eqref{eq:primal dual variable setting} satisfies the stationarity condition \eqref{eq:stationarity w} with probability at least $1 - \calO(\frac{1}{d})$, where $\tau_1(C_{\min}, \alpha, \sigma, \Sigma, \rho)$ is a constant independent of $d, s$ or $n$.  
}
\begin{proof}
	\label{proof:w_s primal dual}
	Consider the following optimization problem:
	\begin{align}
	\label{eq:opt prob 4}
	\begin{split}
	\begin{matrix}
	\min_{w} & \frac{1}{n} (\bX w + \gamma \bz^* - \by)^\T (\bX w + \gamma \bz^* - \by) + \lambda_n \| w \|_1 
	\end{matrix} 
	\end{split}
	\end{align}
	Observe that the above problem is a transformation of optimization problem \eqref{eq:opt prob 2} by fixing $\bZ = \bZ^*$. 	With infinite samples (i.e., $n \to \infty, \lambda_n \to 0$), optimization problem \eqref{eq:opt prob 4} is equivalent to the following population version:
	\begin{align}
	\label{eq:opt prob population}
	\begin{split}
	\begin{matrix}
	\min_{w} & \E((X w + \gamma z^* - y)^\T (X w + \gamma z^* - y))
	\end{matrix} \, .
	\end{split}
	\end{align}
	Clearly, due to Assumption \ref{assum:postive definite}, $w^*$ is the unique optimal solution to \eqref{eq:opt prob population}. 	Let $\tw$ be the solution to the optimization problem \eqref{eq:opt prob 4}. Notice that after replacing $\bZ$ with $\bZ^*$ the stationarity condition \eqref{eq:stationarity w} is same as the stationarity condition for optimization problem \eqref{eq:opt prob 4}:
	\begin{align}
	\label{eq:stationarity w 1}
	\begin{split}
	&\diff{L(w, \bZ; \balpha, \Lambda)}{w} = \mathbf{0}_{d\times 1}
	\end{split}
	\end{align}
	The above simplifies into the following:
	\begin{align*}
	\begin{split}
	\frac{2}{n} \bX^\T \bX \tw - \frac{2}{n} \bX^\T \by + \frac{2\gamma}{n} \bX^\T \bz^* + \lambda_n \bg = \mathbf{0}_{d\times 1}
	\end{split}
	\end{align*}
	Substituting $\by$ from equation \eqref{eq:sample generative process}, we get:
	\begin{align}
	\label{eq:stationarity g}
	\begin{split}
	\frac{2}{n} \bX^\T \bX \Delta - \frac{2}{n} \bX^\T \be + \lambda_n \bg = \mathbf{0}_{d\times 1} \, ,
	\end{split}
	\end{align}
	where $\Delta$ is a short form notation for $\tw - w^*$. To prove our claim, it suffices to show that $\tw = (\tw_S, \bzero_{d-s\times 1})$ satisfies the stationarity condition \eqref{eq:stationarity g}.  
	This will be true iff $\bg_S \in \{-1, 1\}^{s}$ and $\bg_{S^c} \in [-1, 1]^{d-s}$. In particular, if we start with $w = [w_S, \bzero_{d-s \times 1}]$ and show that $\| \bg_{S^c} \|_{\infty} < 1$, then our claim holds. To show this, we replace $w$ with $[w_S, \bzero_{d-s \times 1}]$ and rewrite equation \eqref{eq:stationarity g} in two parts:
	\begin{align}
	\label{eq:g support}
	\begin{split}
	\frac{1}{n} \bX_S^\T \bX_S \Delta_S - \frac{1}{n} \bX_S^\T \be + \frac{\lambda_n}{2} \bg_S = \mathbf{0}_{s\times 1} \, ,
	\end{split}
	\end{align} 
	and
	\begin{align}
	\label{eq:g nonsupport}
	\begin{split}
	\frac{1}{n} \bX_{S^c}^\T \bX_S \Delta_S - \frac{1}{n} \bX_{S^c}^\T \be + \frac{\lambda_n}{2} \bg_{S^c} = \mathbf{0}_{d - s\times 1} \, ,
	\end{split}
	\end{align}
	where $\Delta_S = w_S - w_S^*$. From equation \eqref{eq:g support}:
	\begin{align*}
	\begin{split}
	\Delta_S = (\frac{1}{n} \bX_S^\T \bX_S )^{-1} \frac{1}{n} \bX_S^\T \be - (\frac{1}{n} \bX_S^\T \bX_S )^{-1} \frac{\lambda_n}{2} \bg_S
	\end{split}
	\end{align*} 
	
	By substituting $\Delta_S$ in equation \eqref{eq:g nonsupport}, we get:
	\begin{align*}
	\begin{split}
	&\bHhat_{S^cS} (\bHhat_{SS}^{-1} \frac{1}{n} \bX_S^\T \be - \bHhat_{SS}^{-1} \frac{\lambda_n}{2} \bg_S) - \frac{1}{n} \bX_{S^c}^\T \be + \frac{\lambda_n}{2} \bg_{S^c} = \mathbf{0}_{d - s\times 1}
	\end{split}
	\end{align*}
	By rearranging terms and using the triangle inequality, we get the following:
	\begin{align*}
	\begin{split}
	&\| \frac{\lambda_n}{2} \bg_{S^c} \|_{\infty} \leq \| \bHhat_{S^cS} \bHhat_{SS}^{-1} \frac{1}{n} \bX_S^\T \be \|_{\infty} + \|  \bHhat_{S^cS} \bHhat_{SS}^{-1} \frac{\lambda_n}{2} \bg_S \|_{\infty} + \| \frac{1}{n} \bX_{S^c}^\T \be \|_{\infty} 
	\end{split}
	\end{align*}
	Using the norm inequality $\| A b \|_{\infty} \leq \| A \|_{\infty} \| b \|_{\infty}$ and noticing that $\| \bg_S \|_{\infty} \leq 1$, it follows that:
	\begin{align*}
	\begin{split}
	&\| \frac{\lambda_n}{2} \bg_{S^c} \|_{\infty} \leq \| \bHhat_{S^cS} \bHhat_{SS}^{-1} \|_{\infty} (\| \frac{1}{n} \bX_S^\T \be \|_{\infty} +  \frac{\lambda_n}{2}) + \| \frac{1}{n} \bX_{S^c}^\T \be \|_{\infty} 
	\end{split}
	\end{align*}
	Furthermore, using Lemma \ref{lem:sample mutual incoherence condition}, $\| \bHhat_{S^cS} \bHhat_{SS}^{-1} \|_{\infty} \leq 1 - \frac{\alpha}{2}$ with probability at least $1 - \exp(- \frac{n \tau(C_{\min}, \alpha, \sigma, \Sigma)}{s^2} + \log s)$:
	\begin{align*}
	\begin{split}
	&\| \bg_{S^c} \|_{\infty} \leq (1 - \frac{\alpha}{2}) (\|  \frac{2}{\lambda_n}\frac{1}{n} \bX_S^\T \be \|_{\infty} +  1) + \| \frac{2}{\lambda_n} \frac{1}{n} \bX_{S^c}^\T \be \|_{\infty}
	\end{split}
	\end{align*}
	Next, we will need to bound $\| \frac{1}{n} \bX_S^\T \be \|_{\infty}$ and $\| \frac{1}{n} \bX_{S^c}^\T \be \|_{\infty}$ which we do in the following lemma:
	\begin{lemma}
		\label{lem:bound X_se and X_s^ce}
		Let $\lambda_n \geq \frac{128 \rho k}{\alpha} \frac{\sqrt{\log d}}{n}$ and $n \geq \frac{\log d}{(1 - \frac{\alpha}{2})^2}$. Then the following holds true:
		\begin{align*}
		\begin{split}
		&\prob(\| \frac{2}{\lambda_n} \frac{1}{n} \bX_S^\T \be \|_{\infty} \geq \frac{\alpha}{8 - 4\alpha}) \leq \calO(\frac{1}{d}), \quad \prob(\| \frac{2}{\lambda_n} \frac{1}{n} \bX_{S^c}^\T \be \|_{\infty} \geq \frac{\alpha}{8}) \leq \calO(\frac{1}{d}) 
		\end{split}
		\end{align*}
	\end{lemma}
	
	Using results from Lemma \ref{lem:bound X_se and X_s^ce}, we show that $\| \bg_{S^c} \|_{\infty} \leq 1 - \frac{\alpha}{4} $ with probability at least $1 - \calO(\frac{1}{d})$.
	This ensures that $\tw = (\tw_S, \bzero_{d-s \times 1})$ indeed satisfies the stationarity condition \eqref{eq:stationarity w}.
\end{proof}

\section{Proof of Lemma \ref{lem:bound X_se and X_s^ce}}
\label{appndx:lem:bound X_se and X_s^ce}

\paragraph{Lemma~\ref{lem:bound X_se and X_s^ce}}
\emph{Let $\lambda_n \geq \frac{128 \rho k}{\alpha} \frac{\sqrt{\log d}}{n}$ and $n \geq \frac{\log d}{(1 - \frac{\alpha}{2})^2}$. Then the following holds true:
	\begin{align}
	\begin{split}
	\prob(\| \frac{2}{\lambda_n} \frac{1}{n} \bX_S^\T \be \|_{\infty} &\geq \frac{\alpha}{8 - 4\alpha}) \leq \calO(\frac{1}{d}) \\
	\prob(\| \frac{2}{\lambda_n} \frac{1}{n} \bX_{S^c}^\T \be \|_{\infty} &\geq \frac{\alpha}{8}) \leq \calO(\frac{1}{d}) 
	\end{split}
	\end{align}
}
\begin{proof}
	\label{proof:bound X_se and X_s^ce}
	We will start with $ \frac{1}{n} \bX_S^\T \be$. We take the $i$-th entry of $ \frac{1}{n} \bX_S^\T \be$ for some $i \in S$. Note that
	\begin{align}
	\begin{split}
	| \frac{1}{n} \bX_{i.}^\T \be | = | \frac{1}{n} \sum_{j=1}^n \bX_{ji} \be_j |
	\end{split}
	\end{align}
	Recall that $\bX_{ji}$ is a sub-Gaussian random variable with parameter $\rho^2$ and $\be_j$ is a sub-Gaussian random variable with parameter $\sigma_e^2$. Then, $\frac{\bX_{ji}}{\rho}\frac{\be_j}{\sigma_e}$ is a sub-exponential random variable with parameters $(4\sqrt{2}, 2)$. Using the concentration bounds for the sum of independent sub-exponential random variables~\cite{wainwright2019high}, we can write:
	\begin{align}
	\begin{split}
	\prob( | \frac{1}{n} \sum_{j=1}^n \frac{\bX_{ji}}{\rho}\frac{\be_j}{\sigma_e} | \geq t) \leq 2 \exp(- \frac{nt^2}{64}), \; 0 \leq t \leq 8
	\end{split}
	\end{align} 
	Taking a union bound across $i \in S$:
	\begin{align}
	\begin{split}
	&\prob( \exists i \in S \mid | \frac{1}{n} \sum_{j=1}^n \frac{\bX_{ji}}{\rho}\frac{\be_j}{\sigma_e} | \geq t) \leq 2s \exp(- \frac{nt^2}{64})\\
	&0 \leq t \leq 8
	\end{split}
	\end{align}
	
	Taking $t = \frac{\lambda_n t}{2 \rho \sigma_e}$, we get: 
	\begin{align}
	\begin{split}
	&\prob( \exists i \in S \mid |\frac{2}{\lambda_n} \frac{1}{n} \sum_{j=1}^n \bX_{ji} \be_j | \geq t) \leq 2s \exp(- \frac{n \lambda_n^2 t^2}{256 \rho^2 \sigma_e^2})\\
	&0 \leq t \leq 16 \frac{\rho \sigma_e}{\lambda_n}
	\end{split}
	\end{align}
	
	It follows that $\| \frac{2}{\lambda} \frac{1}{n} \bX_S^\T \be \|_{\infty} \leq  t$ with probability at least $1 - 2s \exp(- \frac{n \lambda_n^2 t^2}{256 \rho^2 \sigma_e^2})$.
	
	Using a similar argument, we can show that $\| \frac{2}{\lambda} \frac{1}{n} \bX_{S^c}^\T \be \|_{\infty} \leq  t$ with probability at least $1 - 2(d - s) \exp(- \frac{n \lambda_n^2 t^2}{256 \rho^2 \sigma_e^2})$. Taking $t= \frac{\alpha}{8 - 4 \alpha}$ and $\frac{\alpha}{8}$ in the first and second inequality of Lemma \ref{lem:bound X_se and X_s^ce} and choosing the provided setting of $\lambda_n$ and $n$ completes our proof.
\end{proof}

\section{Proof of Lemma \ref{lem:bound on Delta}}
\label{appndx:proof lem:bound on Delta}

\paragraph{Lemma \ref{lem:bound on Delta}}
\emph{If Assumptions~\ref{assum:postive definite} and \ref{assum:mutual incoherence condition} hold, $\lambda_n \geq \frac{128 \rho k \sqrt{\log d}}{\alpha n}$  and $n = \Omega( \frac{s^3 \log d}{\tau_2(C_{\min}, \rho, k)})$, then $\| \Delta_S \|_2 \leq \frac{2\lambda_n \sqrt{s}}{C_{\min}}$ with probability at least $1 - \calO(\frac{1}{d})$ where $\tau_2(C_{\min}, \rho, k)$ is a constant independent of $s, d$ or $n$.}
\begin{proof}
	\label{proof:bound on Delta}
	Using results from Lemma~\ref{lem:w_s primal dual}, we can write:
	\begin{align*}
	\begin{split}
	\| \Delta_S \|_2 \leq \| \bHhat_{SS}^{-1} \frac{1}{n} \bX_S^\T \be \|_2 + \| \bHhat_{SS}^{-1} \frac{\lambda_n}{2} \bg_S \|_2
	\end{split}
	\end{align*}
	Using the norm inequality $\| A b\|_2 \leq \| A \|_2 \| b \|_2$ and noticing that $\| \bg_S \|_2 \leq \sqrt{s}$, we can rewrite the above equation as:
	\begin{align*}
	\begin{split}
	\| \Delta_S \|_2 \leq \| \bHhat_{SS}^{-1} \|_2 (\| \frac{1}{n} \bX_S^\T \be \|_2 +  \frac{\lambda_n}{2} \sqrt{s} )
	\end{split}
	\end{align*}   
	Using Assumption~\ref{assum:postive definite} and results from Lemma \ref{lem:sample positive definite} and substituting $\| \bHhat_{SS}^{-1} \|_2 \leq \frac{2}{C_{\min}}$  in the above inequality, we get:
	\begin{align*}
	\begin{split}
	\| \Delta_S \|_2 \leq \frac{2}{C_{\min}} (\| \frac{1}{n} \bX_S^\T \be \|_2 +  \frac{\lambda_n}{2} \sqrt{s} )
	\end{split}
	\end{align*}
	We present the next lemma to bound the term $\| \frac{1}{n} \bX_S^\T \be \|_2$.
	\begin{lemma}
		\label{lem:bound on l2 Xse}
		If  $\lambda_n \geq \frac{128 \rho k}{\alpha} \frac{\sqrt{\log d}}{n}$ and $n = \Omega( \frac{s^3 \log d}{\tau_2(C_{\min}, \rho, k)})$, then $\| \frac{1}{n} \bX_S^\T \be \|_2 \leq \sqrt{s} \frac{\lambda_n}{2}$ with probability at least $1 - \calO(\frac{1}{d})$.
	\end{lemma}
	We take $t = \frac{\lambda_n}{2}$ in the above lemma and get $ \| \Delta \|_2 \leq \frac{2 \lambda_n \sqrt{s} }{C_{\min}}  $ with probability at least $1 - \calO(\frac{1}{d})$. 
\end{proof}

\section{Proof of Lemma \ref{lem:bound on l2 Xse}}
\label{appndx:lem:bound on l2 Xse}
\paragraph{Lemma \ref{lem:bound on l2 Xse}}
\emph{If  $\lambda_n \geq \frac{128 \rho k}{\alpha} \frac{\sqrt{\log d}}{n}$ and $n = \Omega( \frac{s^3 \log d}{\tau_2(C_{\min}, \rho, k)})$, then $\| \frac{1}{n} \bX_S^\T \be \|_2 \leq \sqrt{s} \frac{\lambda_n}{2}$ with probability at least $1 - \calO(\frac{1}{d})$.
}
\begin{proof}
	\label{proof:bound on l2 Xse}
	We take the $i$-th entry of $ \frac{1}{n} \bX_S^\T \be$ for some $i \in S$. Note that
	\begin{align}
	\begin{split}
	| \frac{1}{n} \bX_{i.}^\T \be | = | \frac{1}{n} \sum_{j=1}^n \bX_{ji} \be_j |
	\end{split}
	\end{align}
	Recall that $\bX_{ji}$ is a sub-Gaussian random variable with parameter $\rho$ and $\be_j$ is a sub-Gaussian random variable with  parameter $\sigma_e^2)$. Then, $\frac{\bX_{ji}}{\rho}\frac{\be_j}{\sigma_e}$ is a sub-exponential random variable with parameters $(4\sqrt{2}, 2)$. Using the concentration bounds for the sum of independent sub-exponential random variables~\cite{wainwright2019high}, we can write:
	\begin{align}
	\begin{split}
	\prob( | \frac{1}{n} \sum_{j=1}^n \frac{\bX_{ji}}{\rho}\frac{\be_j}{\sigma_e} | \geq t) \leq 2 \exp(- \frac{nt^2}{64}), \; 0 \leq t \leq 8
	\end{split}
	\end{align} 
	Taking a union bound across $i \in S$, we get
	\begin{align}
	\begin{split}
	&\prob( \exists i \in S \mid | \frac{1}{n} \sum_{j=1}^n \frac{\bX_{ji}}{\rho}\frac{\be_j}{\sigma_e} | \geq t) \leq 2s \exp(- \frac{nt^2}{64}),\\
	& 0 \leq t \leq 8
	\end{split}
	\end{align}
	It follows that $\| \frac{1}{n} \bX_S^\T \be \|_2 \leq \sqrt{s} t$ with probability at least $1 - 2s \exp(- \frac{nt^2}{64 \rho^2\sigma_e^2})$ for some $0 \leq t \leq 8 \rho \sigma_e$.
\end{proof}

\section{Proof of Corollary \ref{cor:primal dual vairables}}
\label{appndx:proof of cor:primal dual vairables}

\paragraph{Corollary \ref{cor:primal dual vairables}}
\emph{If Assumptions~\ref{assum:postive definite} and \ref{assum:mutual incoherence condition} hold, $\lambda_n \geq \frac{128 \rho k}{\alpha} \frac{\sqrt{\log d}}{n}$ and $n = \Omega( \frac{s^3 \log d}{\tau_1(C_{\min}, \alpha, \sigma, \Sigma, \rho)} )$, then the following statements are true with probability at least $1 - \calO(\frac{1}{n})$:
	\begin{enumerate}
		\item The solution $\bZ$ correctly recovers hidden attribute for each sample, i.e., $\bZ = \bZ^* = \zeta^* {\zeta^*}^\T$.
		\item The support of recovered regression parameter $\tw$ matches exactly with the support of $w^*$.
		\item If $\min_{i \in S} |w_i^*| \geq  \frac{4 \lambda_n \sqrt{s} }{C_{\min}} $ then for all $i \in \seq{d}$, $\tw_i$ and $w_i^*$ match up to their sign. 
	\end{enumerate}
}
\begin{proof}
	\label{proof:primal dual vairables corollary}
	Since $\bZ = \bZ^*$, the hidden attributes of each sample can be read by simply looking at the first row or column of $\bZ$ and skipping the first entry. The supports of $\tw$ and $w^*$ match exactly through construction (and subsequent proofs). Observe that $\| \Delta \|_{\infty} \leq \| \Delta \|_2 \leq \frac{2 \lambda_n \sqrt{s} }{C_{\min}} $. Thus, it follows that if $\min_{i \in S} |w_i^*| \geq  \frac{4 \lambda_n \sqrt{s} }{C_{\min}} $ then for all $i \in \seq{d}$, $\tw_i$ and $w_i^*$ will have the same sign.
\end{proof}

\section{Quality of Solution with bias parameter $\gamma$}
\label{appndx:recovery with gamma}

\begin{figure*}[!ht]
	\centering
	\begin{subfigure}{.45\textwidth}
		\centering
		\includegraphics[width=\linewidth]{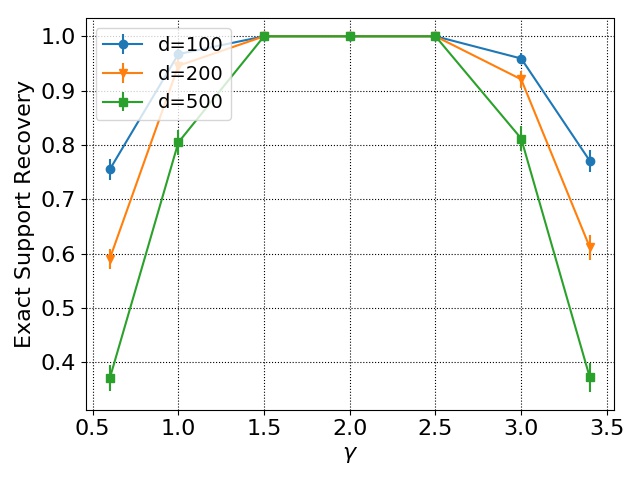}
		\caption{Recovery of $S$ versus $\gamma$}
		\label{fig:gamma_supp}
	\end{subfigure}%
	\begin{subfigure}{.45\textwidth}
		\centering
		\includegraphics[width=\linewidth]{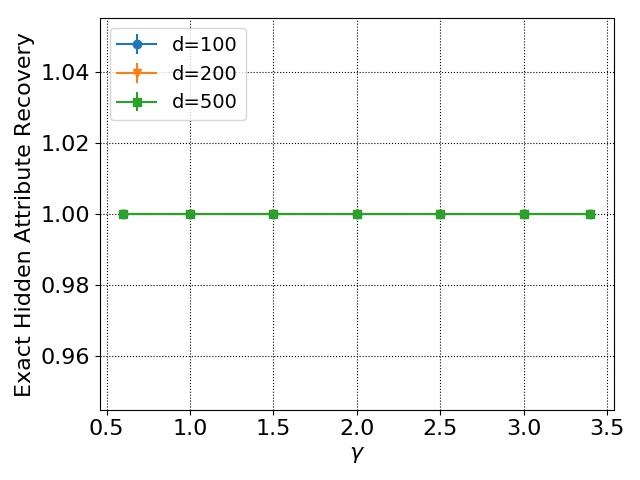}	
		\caption{Recovery of $\bZ^*$ versus $\gamma$}
		\label{fig:gamma_label}
	\end{subfigure}%
	\caption{Left: Exact support recovery of $w^*$ across $30$ runs. Right: Exact hidden attribute recovery of $\bZ*$ across $30$ runs. The true value of $\gamma$ is $2$.}
	\label{fig:gamma}
\end{figure*}

Our method requires a known value of bias parameter $\gamma$ in our analysis. However, in practice, we observe that even a rough estimate (up to $\pm 25\%$) works pretty well. We conducted computational experiments with a range of values of $\gamma$ and the reported results are averaged across $30$ independent runs. The performance measures used here are the same as in Section~\ref{sec:experimental results} (See Appendix~\ref{appndx:experimental results} for details). Figure~\ref{fig:gamma_supp} shows the quality of support recovery for different values of $\gamma$ and Figure~\ref{fig:gamma_label} shows the quality of recovering the hidden attributes for different values of $\gamma$. Note how both the curves show $100\%$ correct recovery for a wide range of $\gamma$. These experiments show that prior knowledge of the exact value of $\gamma$ is not necessary for our method.

\section{Alternate Optimization Algorithm for Solving Optimization Problem \eqref{eq:opt prob 2}}
\label{appndx:alternate optimization}

We use the following alternate optimization algorithm to solve optimization problem~\eqref{eq:opt prob 2} in our computational experiments. 

\begin{algorithm}[!h]
	\begin{algorithmic}
		\STATE \textbf{Input:} Data samples $(\bX, \by)$, amount of bias $\gamma$
		\STATE \textbf{Output:}  $\tw, \bZ$
		\STATE $\bZ_0 \gets \mathbf{I}_{n+1 \times n+1}$
		\STATE $\bz_0 \gets \bZ_0(2:n+1, 1)$ 
		\FOR{$t=1, 2,\cdots$ until $\bZ_{t-1} = \bZ_{t}$}
		\STATE{
			$\tw_t \gets \arg\min_{w} \frac{1}{n} (\bX w + \gamma \bz_{t-1} - \by)^\T (\bX w + \gamma \bz_{t-1} - \by) + \lambda_n \| w \|_1 $ \\
			$\bM(\tw_t) \gets \begin{bmatrix} \frac{1}{n} \| \bX\tw_t - \by \|_2^2 & \frac{\gamma}{n} (\bX \tw_t - \by)^\T \\ \frac{\gamma}{n} (\bX \tw_t - \by)^\T & \frac{\gamma^2}{n} \bI_{n \times n}  \end{bmatrix} $\\ 
			$\bZ_t \gets \arg\min_{\bZ} {\tr(\bM(\tw_t)\bZ)}$, such that $\diag(\bZ) = \mathbf{1}$, $\bZ \succeq \bzero_{n+1\times n+1}$\\
			$\bz_t \gets \bZ_t(2:n+1, 1)$   
		}
		\ENDFOR
		\STATE $\tw \gets \tw_t,\;\; \bZ \gets \bZ_t$
	\end{algorithmic}
	\caption{\label{algo:alternate_opt} Alternate Optimization Algorithm for our problem}
\end{algorithm}%

Recall from equation~\eqref{eq:M and Z} that 
\begin{align}
\begin{split}
\bZ \triangleq \begin{bmatrix} 1 & \bz^\T \\ \bz & \bz \bz^\T  \end{bmatrix}.
\end{split}
\end{align}
Thus, we can read $\bz$ from $\bZ$ by considering its first column and skipping the first entry. We denote this as $\bz = \bZ(2:n+1, 1)$. We use a similar notation in Algorithm~\ref{algo:alternate_opt} to assign values to vectors $\bz_0$ and $\bz_t$ from matrices $\bZ_0$ and $\bZ_t$ respectively.

We will show that if Algorithm~\ref{algo:alternate_opt} converges then it converges to the optimal solution of optimization problem \eqref{eq:opt prob 2}. To do this, consider 
\begin{align}
\begin{split}
f_1(w, \bZ) &= \frac{1}{n} (\bX w + \gamma \bz_{t-1} - \by)^\T (\bX w + \gamma \bz_{t-1} - \by) \\
f_2(w) &= \lambda_n \| w \|_1 \; .
\end{split}
\end{align}  
Note that $f_2(w)$ is not differentiable. Let $g(\bZ) \triangleq - \eig_{\min}(\bZ) $  and $h_i(\bZ) \triangleq \bZ_{ii} - 1, \forall i \in \seq{n+1}$. Observe that $g(\bZ) \leq 0$ and $h_i(\bZ) = 0, \forall i \in \seq{n+1}$ denote the constraints $\bZ \succeq \bzero_{n+1 \times n+1}$ and $\diag(\bZ) = \mathbf{1}$ respectively. We define $\diff{f_2(w)}{w}$ as the sub-differential set for $f_2(w)$ and $f_2'(w) \in \diff{f_2(w)}{w}$ is an element of the sub-differential set $\diff{f_2(w)}{w}$. Observe that $f_1(w, \bZ) + f_2(w)$, $g(\bZ)$ and $h_i(\bZ)$ are convex with respect to $w$ and $\bZ$ separately but they are not jointly convex. Consider the following optimization problem:
\begin{align}
\label{eq:opt prob alternate}
\begin{split}
\tw, \bZ^* = \begin{matrix}
\arg\min_{w, \bZ} &  f_1(w, \bZ) + f_2(w) & \\
\text{such that} & g(\bZ) \leq 0 &  \\
& h_i(\bZ) = 0 & \forall i \in \seq{n+1} 
\end{matrix}
\end{split}
\end{align}

We have already shown that the solution $\tw, \bZ^*$ is the unique solution to \eqref{eq:opt prob alternate}. We propose the following alternate optimization algorithm to solve this problem:  

\begin{algorithm}[!h]
	\begin{algorithmic}
		\STATE \textbf{Output:}  $w, \bZ$
		\STATE $\bZ_0 \gets \mathbf{I}_{n+1 \times n+1}$
		\FOR{$t=1, 2\cdots$ until $\bZ_{t-1} = \bZ_{t}$}
		\STATE{
			\begin{align}
			\label{eq:yt} 
			w_t \gets \arg\min_{w} f_1(w, \bZ_{t-1}) + f_2(w)
			\end{align} 
			\begin{align}
			\label{eq:xt}
			\bZ_t \gets \begin{matrix}
			\arg\min_{\bZ} &  f_1(w_t, \bZ) & \\
			\text{such that} & g(\bZ) \leq 0 & \\
			& h_i(\bZ) = 0 & \forall i \in \seq{n+1} 
			\end{matrix}
			\end{align}   
		}
		\ENDFOR
		\STATE $w \gets w_t,\;\; \bZ \gets \bZ_t$
	\end{algorithmic}
	\caption{\label{algo:alternate_opt_gen} Alternate Optimization Algorithm}
\end{algorithm}

We will prove the following proposition:
\begin{proposition}
	\label{prop:alternate opt}
	If Algorithm \ref{algo:alternate_opt_gen} converges, then $w = \tw$ and $\bZ = \bZ^*$. 
\end{proposition}
\begin{proof}
	We start by writing the KKT conditions for optimization problem \eqref{eq:opt prob alternate}.
	\begin{enumerate}
		\item Stationarity conditions: $\diff{f_1(\tw, \bZ^*)}{w} + f_2'(\tw)  = 0$ and $\diff{f_1(\tw, \bZ^*)}{\bZ} + r \diff{g(\bZ^*)}{\bZ} +  \sum_{i=1}^{n+1} s_i \diff{h_i(\bZ^*)}{\bZ}  = 0$.
		\item Complementary slackness condition: $r g(\bZ^*) = 0$.
		\item Primal feasibility condition: $g(\bZ^*) \leq 0$ and $h_i(\bZ^*) = 0, \forall i \in \seq{n+1}$.
		\item Dual feasibility condition: $r \geq 0$.   
	\end{enumerate}
	Any optimal solution to optimization problem \eqref{eq:opt prob alternate} must satisfy the above KKT conditions. Next, we write the KKT conditions for \eqref{eq:yt} at convergence, i.e., at $\bZ_{t} = \bZ_{t-1}$:
	\begin{enumerate}
		\item Stationarity condition: $\diff{f(w_t, \bZ_t)}{w} + f_2'(w_t) = 0$
	\end{enumerate}
	Similarly, we write the KKT conditions for \eqref{eq:xt} at convergence, i.e., at $\bZ_{t} = \bZ_{t-1}$:
	\begin{enumerate}
		\item Stationarity conditions: $\diff{f_1(w_t, \bZ_t)}{\bZ} +  t \diff{g(\bZ_t)}{\bZ} +  \sum_{i=1}^{n+1} u_i \diff{h_i(\bZ_t)}{\bZ}  = 0$.
		\item Complementary slackness condition: $t g(\bZ_t) = 0$.
		\item Primal feasibility condition: $g(\bZ_t) \leq 0$ and $h_i(\bZ_t) = 0, \forall i \in \seq{n+1}$.
		\item Dual feasibility condition: $t \geq 0$.   
	\end{enumerate}
	Combining the KKT conditions at $w_t, \bZ_t$ for \eqref{eq:yt} and \eqref{eq:xt} and taking $r = t$ and $s_i = u_i, \forall i \in \seq{n+1}$, we see that all KKT conditions of \eqref{eq:opt prob alternate} are satisfied by $w_t, \bZ_t$. Since the solution to \eqref{eq:opt prob alternate} is unique, it follows that $w  = \tw$ and $\bZ = \bZ^*$.
\end{proof}

\section{Our Assumptions Hold for Finite Samples}
\label{appndx:assumptions hold in sample setting}

\begin{figure*}[!ht]
	\centering
	\begin{subfigure}{.45\textwidth}
		\centering
		\includegraphics[width=\linewidth]{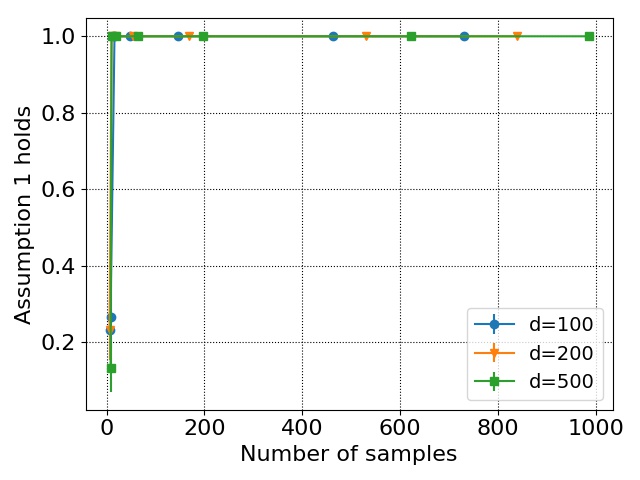}
		\caption{Positive Definiteness against number of samples}
		\label{fig:assum1}
	\end{subfigure}%
	\begin{subfigure}{.45\textwidth}
		\centering
		\includegraphics[width=\linewidth]{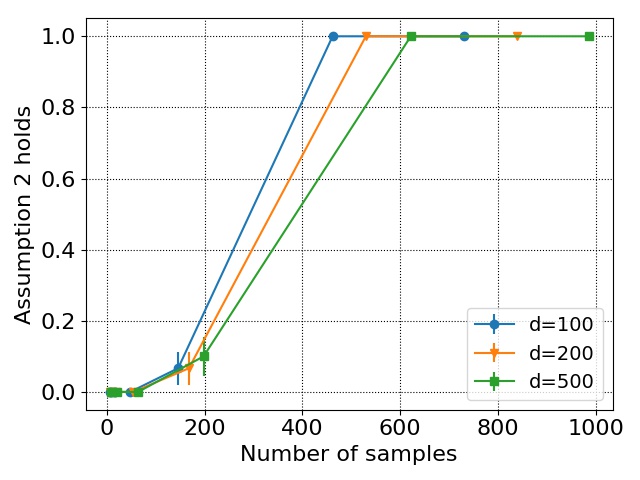}	
		\caption{Mutual Incoherence against number of samples}
		\label{fig:assum2}
	\end{subfigure}%
	\caption{Left: Positive Definiteness Assumption~\ref{assum:postive definite} with varying number of samples for $d=100, 200$ and $500$, Right: Mutual Incoherence Assumption~\ref{assum:mutual incoherence condition} with varying number of samples for $d=100, 200$ and $d=500$.}
	\label{fig:assum}
\end{figure*}  

Figure~\ref{fig:assum} shows how our assumptions hold (averaged across 30 independent runs) in the finite sample regime with varying number of samples  when $X$ is drawn from a standard normal distribution.  We notice that for a fixed $s$, Assumption~\ref{assum:postive definite} is easier to hold (i.e., $n = \Omega(s + \log d)$) than Assumption~\ref{assum:mutual incoherence condition} (i.e., $n = \Omega(s^3 \log d)$). Eventually, both assumptions hold as the number of samples increases.

\section{Details of Experimental Validation}
\label{appndx:experimental results}

In this section, we validate our theoretical results by conducting computational experiments on synthetic data. We will show that for a fixed $s$, we need $n = 10^\beta \log d$ samples for recovering the exact support of $w^*$ and exact hidden attributes $\bZ^*$, where $\beta \equiv \beta(s, C_{\min}, \alpha, \sigma, \Sigma, \rho, \gamma, k)$ is a control parameter which is independent of $d$.
\paragraph{Data Generation.} For $d=100, 200$ and $500$, we draw $\bX \in \real^{n \times d}$ from a standard Gaussian distribution by varying $n$ as $10^\beta \log d$ for a control parameter $\beta$. The $s=10$ non-zero entries of true parameter $w^* \in \real^d$ are chosen uniformly at random between $[-1, 1]$. Every non-zero entry in $w^*$ is changed to at least $0.75$ to make sure that it is not too close to $0$. The independent noise $\be \in \real^n$ is drawn from a zero mean Gaussian distribution with standard deviation $\frac{k}{\sqrt{\log n}}$ for $k = 0.15$. The estimate of the bias $\gamma \in \real_{> 0}$ is kept at $2$. Regarding the hidden attribute $\bz^* \in \{-1,1\}^n$, we set $\frac{n}{2}$ entries as $+1$ and the rest as $-1$. The response $\by \in \real^n$ is generated according to \eqref{eq:generative process}. This process is repeated $30$ times and the reported results are averaged across these $30$ independent runs. 

\paragraph{Choice of Regularizer and Solution.}  According to Theorem~\ref{thm:primal dual witness construction}, the regularizer $\lambda_n$ is chosen to be equal to $\frac{128 \rho k}{\alpha} \frac{\sqrt{\log d}}{n}$. We solve optimization problem \eqref{eq:opt prob 2} by using an alternate optimization algorithm that converges to the optimal solution (See Appendix \ref{appndx:alternate optimization} for details).

\paragraph{Measure of Performance.} The performance is measured by comparing the recovered solutions $\tw$ and $\bZ$ with the true parameters $w^*$ and $\bZ^*$. The quality of $\tw$ is measured by comparing its support to the support $S$ of the true parameter $w^*$ by computing the Jaccard index $J(S, \hat{S})$, where $\hat{S}$ is the support of $\tw$, i.e., $\hat{S} = \{ i | \tw_i \ne 0, i \in \seq{d} \}$. The average of $J(S, \hat{S})$ across 30 independent runs is plotted against the number of samples $n$ (See Figure~\ref{fig:recnumsample}, \ref{fig:recnumsamplecp}). Similarly, the quality of $\bZ$ is measured by the indicator variable $I(\bZ, \bZ^*)$. The average of $I(\bZ, \bZ^*)$ across 30 independent runs is plotted against the number of samples $n$ (See Figure~\ref{fig:labelnumsample}, \ref{fig:labelnumsamplecp}). The Jaccard index $J(S, \hat{S})$ and indicator variable $I(\bZ, \bZ^*)$ are defined as follows:
\begin{align*}
\begin{split}
J(S, \hat{S}) \triangleq \frac{|S \cap \hat{S}|}{|S \cup \hat{S}|}, \;\;\; 	I(\bZ, \bZ^*) \triangleq \begin{cases}
0, \; \rm{if}\; \bZ \ne \bZ^* \\
1 \; \rm{if}\; \bZ = \bZ^*
\end{cases}
\end{split}
\end{align*} 

\paragraph{Observation.} Figure~\ref{fig:recnumsample} shows the Jaccard index of support recovery with varying number of samples. We see that our method recovers the true support for all three values of $d$ as we increase number of samples. Also, notice how all three curves line up perfectly in Figure~\ref{fig:recnumsamplecp} when we plot the support recovery with respect to the control parameter $\beta = \log \frac{n}{\log d}$. This validates our theoretical results. Similarly, Figure~\ref{fig:labelnumsample} shows exact recovery of the hidden attribute with varying number of samples. We again see that as the number of samples increase, our recovered hidden attributes are 100\% correct. Again, the three different curves for different values of $d$ line up nicely when plotted against $\beta$. Interestingly, a small percentage of our experiments recover the hidden attributes exactly for small number of samples ($<20$). We believe that this can be ascribed to $\bZ^*$ having small dimensions and thus becoming relatively easier to recover. On a more practical point of view, once hidden attributes are identified for each sample point, the associated bias (for and against) can be duly removed from the model.   

\section{Optimization Problem~\eqref{eq:opt prob 2} is Non-Convex}
\label{appndx:opt non-convex}

Before we begin the proof of non-convexity of \eqref{eq:opt prob 2}, we note that optimization \eqref{eq:opt prob 2} is stated for a fixed $\bX$. However, since the entries in $\bX$ are drawn from a sub-Gaussian distribution, matrix $\bX$ can potentially realize any real matrix in $\real^{n \times d}$. In particular, we are interested in a problem where $\exists i, k \in \seq{d}$ such that $\sum_{l=1}^n \bX_{li}^2  - \bX_{ki}$ is non-zero. Since $\bX$ can be any real matrix in $\real^{n \times d}$, this is not a strong assumption. With this assumption in mind, we present the following lemma. 

\begin{lemma}
	\label{lem:non-convex opt}
	The optimization problem \eqref{eq:opt prob 2} defined on a convex set $C$, is non-convex.
\end{lemma}
\begin{proof}
	As defined in \eqref{eq:opt prob 2}, we define the domain for optimization problem on a convex set $C = \{ (w, \bZ) \mid w \in \real^d, \diag(\bZ) = \mathbf{1}, \bZ \succeq \mathbf{0}_{n+1 \times n+1} \}$. It should be noted that $C$ is a convex set and we will show that the non-convexity of the problem comes from the objective function. We are solving the following optimization problem:
	\begin{align}
	\begin{split}
	\begin{matrix}
	\min_{(w, \bZ) \in C } & \inner{\bM(w)}{\bZ} + \lambda_n \| w \|_1  \\
	\end{matrix} \, ,
	\end{split}
	\end{align}
	
	It suffices to show that $f(w, \bZ) = \inner{\bM(w)}{\bZ}$ is non-convex function. To that end, we will construct a setting of $(w, \bZ) \in C$ and $(\bar{w}, \bar{\bZ}) \in C$ such that the first order condition for convexity fails to hold, i.e, 
	\begin{align}
	\label{eq:non-convexity}
	\begin{split}
	f(w, \bZ) - f(\bar{w}, \bar{\bZ}) < \inner{\diff{f(\bar{w}, \bar{\bZ})}{w}}{(w - \bar{w})} + \inner{ \diff{f(\bar{w}, \bar{\bZ})}{\bZ}}{ \bZ - \bar{\bZ}} . 
	\end{split}
	\end{align}
	
	First notice that,
	\begin{align*}
	\begin{split}
	\diff{f(\bar{w}, \bar{\bZ})}{w} = \sum_{ij} \bar{\bZ}_{ij} \diff{\bM_{ij}(\bar{w})}{w}, \;\;\; \diff{f(\bar{w}, \bar{\bZ})}{\bZ} = \bM(\bar{w})
	\end{split}
	\end{align*}
	
	Recall from equation~\eqref{eq:M and Z} that, 
	\begin{align}
	\begin{split}
	&l(w) \triangleq \frac{1}{n} (\bX w - \by)^\T (\bX w - \by),\;\; \bM(w) \triangleq \begin{bmatrix} l(w) & \frac{\gamma}{n} (\bX w - \by)^\T \\ \frac{\gamma}{n} (\bX w - \by) & \frac{\gamma^2}{n} \bI_{n \times n}  \end{bmatrix},
	\end{split}
	\end{align} 
	
	
	Then $\diff{f(\bar{w}, \bar{\bZ})}{w}$ can be simplified as:
	\begin{align}
	\begin{split}
	\diff{f(\bar{w}, \bar{\bZ})}{w} = \frac{2}{n}(\bX^\T \bX \bar{w} - \bX^\T \by + \bX^\T \bar{\bz}),
	\end{split}
	\end{align}
	
	where $\bar{\bz} \in \real^n$ denotes the first column of $\bar{\bZ}$ after skipping the first entry.
	
	We provide the following construction for $(w, \bZ) \in C$ and $(\bar{w}, \bar{\bZ}) \in C$. We take $w \in \{0, \beta \}^d$ such that $w_k = 0, \forall k \ne i$ and $w_i = \beta$ where $\beta \in \real$. Similarly,  $\bar{w} \in \{0, \beta \}^d$ such that $\bar{w}_k = 0, \forall k \ne i$ and $\bar{w}_i = - \beta$. Since $w \in \real^d$, such a setting exists for a non-zero $\beta$. Furthermore, we take $\bZ = \mathbf{I}_{n+1 \times n+1}$ and $\bar{\bZ} \in \{0, 1\}^{n+1 \times n+1}$ such that $\bar{\bZ}_{ii} = 1, \forall i \in \seq{n+1}$ and $\bar{\bZ}_{1 (k+1)} = 1, \bar{\bZ}_{(k+1) 1} = 1$. Now, we can compute the following quantities:
	\begin{align}
	\begin{split}
	\inner{\bM(w)}{\bZ} &= l(w) + \gamma^2 = \frac{1}{n} \sum_{l=1}^n (\bX_{li} w_i - y_l)^2 + \gamma^2 \\ 
	\inner{\bM(\bar{w})}{\bar{\bZ}} &= l(\bar{w}) + \gamma^2 - \frac{2\gamma}{n} (\bX_{ki} \bar{w}_i - y_k) = \frac{1}{n} \sum_{l=1}^n (\bX_{li} \bar{w}_i - y_l)^2 + \gamma^2 + \frac{2\gamma}{n} (\bX_{ki} \bar{w}_i - y_k) \\ 
	\inner{\bM(\bar{w})}{\bZ - \bar{\bZ}} &= - \frac{2\gamma}{n} (\bX_{ki} \bar{w}_i - y_k) \\
	\inner{\diff{f(\bar{w}, \bar{\bZ})}{w}}{w - \bar{w}} &= \frac{2}{n} ((w_i \bar{w}_i - \bar{w}_i^2) \sum_{l=1}^n \bX_{li}^2 + (-w_i + \bar{w}_i) \sum_{l=1}^n \bX_{li} \by_l + (w_i - \bar{w}_i) \bX_{ki}) 
	\end{split}
	\end{align}
	Substituting $w_i = \beta$ and $\bar{w}_i = -\beta$, we get
	\begin{align}
	\begin{split}
	l(w) - l(\bar{w}) &= - \frac{4 \beta}{n} \sum_{l=1}^n \bX_{li} \by_l \\
	\inner{\diff{f(\bar{w}, \bar{\bZ})}{w}}{w - \bar{w}} &= - \frac{4\beta}{n} \sum_{l=1}^n \bX_{li}^2 - \frac{4\beta}{n} \sum_{l=1}^n \bX_{li} \by_l + \frac{4\beta}{n} \bX_{ki} 
	\end{split}
	\end{align}
	
	Clearly, 
	\begin{align}
	\begin{split}
	\inner{\bM(w)}{\bZ} - \inner{\bM(\bar{w})}{\bar{\bZ}} &= - \frac{4 \beta}{n} \sum_{l=1}^n \bX_{li} \by_l - \frac{2\gamma}{n} (\bX_{ki} \bar{w}_i - y_k) \\
	\inner{\bM(\bar{w})}{\bZ - \bar{\bZ}} + \inner{\diff{f(\bar{w}, \bar{\bZ})}{w}}{w - \bar{w}} &= - \frac{2\gamma}{n} (\bX_{ki} \bar{w}_i - y_k) - \frac{4\beta}{n} \sum_{l=1}^n \bX_{li}^2 - \frac{4\beta}{n} \sum_{l=1}^n \bX_{li} \by_l + \frac{4\beta}{n} \bX_{ki} 
	\end{split}
	\end{align}
	It follows that
	\begin{align}
	\label{eq:nonconvex beta}
	\begin{split}
	\inner{\bM(w)}{\bZ} - \inner{\bM(\bar{w})}{\bar{\bZ}} - \inner{\bM(\bar{w})}{\bZ - \bar{\bZ}} -\inner{\diff{f(\bar{w}, \bar{\bZ})}{w}}{w - \bar{w}} = \beta (  \frac{4}{n} \sum_{l=1}^n \bX_{li}^2  - \frac{4}{n} \bX_{ki} )
	\end{split}
	\end{align}
	As $\sum_{l=1}^n \bX_{li}^2  -  \bX_{ki}$ is assumed to be non-zero, it is easy to see that LHS of equation~\eqref{eq:nonconvex beta}can be made greater than or less than $0$ by simply choosing appropriate $\beta \in \real$. Thus, optimization problem \eqref{eq:opt prob 2} is non-convex.
	%
	%
\end{proof}

\section{Real World Experiment}
\label{appndx:real world experiments}

We show applicability of our method by conducting experiments on Communities and Crime Data Set~\cite{redmond2002data} and Student Performance Data Set~\cite{cortez2008using}. 

\subsection{Communities and Crime Data Set}
\label{appndx:communities and crime}
This data set contains $1994$ samples with $122$ predictors which might have plausible connection to crime, and the attribute to be predicted (Per Capita Violent Crimes). In the preprocessing step, any predictors with missing values are removed and all the predictors and the attribute to be predicted are standardized to have zero mean and unit standard deviation. The preprocessed dataset contains $d=100$ predictors and $n=1994$ samples.

The optimization problem \eqref{eq:opt prob 2} is solved for $\lambda_n = 0.15$ and $\gamma$ is chosen to be $\frac{\max(\by) - \min(\by)}{2}$. As the problem is invex, any algorithm which converges to a stationary point can be used to solve the problem. We used an alternate optimization algorithm (See Appendix \ref{appndx:alternate optimization}) which converges to an optimal solution.

\paragraph{Main results.} Based on the support (non-zero entries) in the recovered $w$, we found that the following are the most important predictors of Per Capita Violent Crimes:
\begin{enumerate}
	\item PctHousNoPhone: percentage of occupied housing units without phone
	\item PctNotHSGrad: percentage of people 25 and over that are not high school graduates
	\item PctLess9thGrade: percentage of people 25 and over with less than a 9th grade education
	\item RentLowQ: rental housing - lower quartile rent  
\end{enumerate} 

We also recovered the hidden sensitive attribute with $816$ instances of positive bias ($z = +1$) with mean crime rate $0.8002$ and $1178$ instances of negative bias ($z = -1$) with mean crime rate $-0.5543$. By plotting data with two of the most important predictors (PctHousNoPhone, PctNotHSGrad), we clearly see the existence of two groups (Figure \ref{fig:realworld}). Our Mean Squared Error (MSE) is $0.0265$. \cite{chzhen2020fair} can be checked for comparison with other state-of-the-art methods ($12$ methods of $3$ different types) where only the Kernel Regularized Least Square method (MSE=$0.024 \pm 0.003$) and the Random Forests method (MSE=$0.020 \pm 0.002$) perform better than our method in terms of MSE but suffer heavily in terms of fairness. Other methods incur MSE in the range between $0.028 \pm 0.003$ to $0.041 \pm 0.004$.

\begin{figure*}[!ht]
	\centering
	\includegraphics[width=0.8\linewidth]{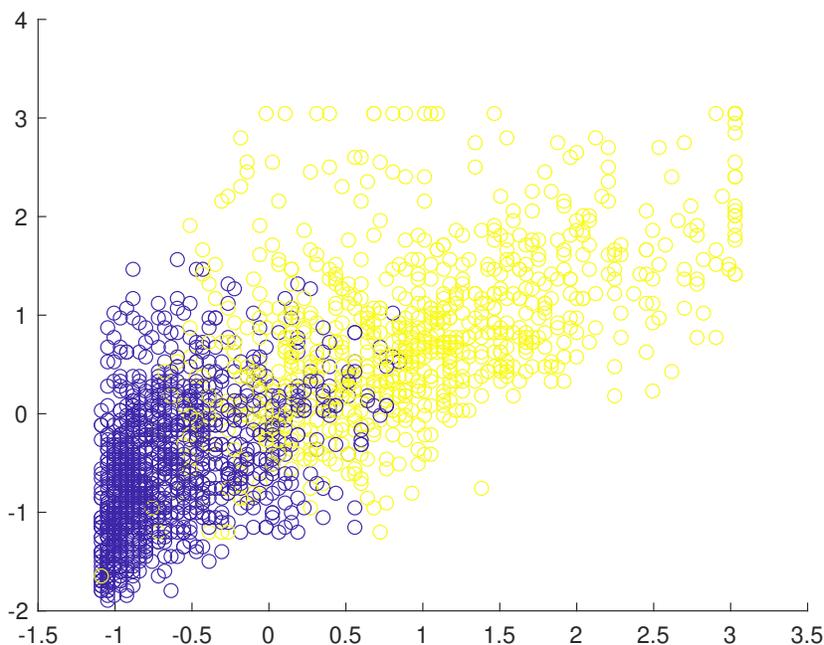}
	\caption{Clusters in Communities and Crime Dataset}
	\label{fig:realworld}
\end{figure*}  

\subsection{Student Performance Data Set}
\label{appndx:student performance data set}
This data set contains $649$ samples with $33$ demographic, social and school predictors and the attribute to be predicted (grade in the Portuguese Language course). The data set contains some categorical variables which are converted to numerical variables using dummy encoding (thus increasing the number of predictors). Two columns containing partial grades were removed from the data set. In the preprocessing step, all the predictors and the attribute to be predicted are standardized to have zero mean and unit standard deviation. The preprocessed dataset contains $d=39$ predictors and $n=649$ samples.

Similar to subsection~\ref{appndx:communities and crime}, the optimization problem \eqref{eq:opt prob 2} is solved for $\lambda_n = 0.15$ and $\gamma = \frac{\max(\by) - \min(\by)}{2}$.

\paragraph{Main results.} The following are the most important predictors of grades in the Portuguese Language course:
\begin{enumerate}
	\item school: student's school
	\item failures: number of past class failures
	\item higher: wants to take higher education
\end{enumerate} 

We also recovered the hidden sensitive attribute with $420$ instances of positive bias ($z = +1$) with mean grade $0.2305$ and $229$ instances of negative bias ($z = -1$) with mean grade $-0.4227$. Our Mean Squared Error (MSE) is $0.0494$. \cite{chzhen2020fair} can be checked for comparison with other state-of-the-art methods ($12$ methods of $3$ different types) where none of the methods performs better than our method in terms of MSE (range between $3.59 \pm 0.39$ to $5.62 \pm 0.52$).

\subsection{Discussion.} While our analysis identifies two groups with bias in both data sets, it cannot only be attributed to the most important recovered  predictors. Recall the ``red-lining'' effect~\cite{calders2010three} where there might be other correlated predictors which can facilitate indirect discrimination. For example: in the Communities and Crime data set, annual income could be correlated with PctHousNoPhone and similarly in the Student Performance data set, parents' educational qualification could be correlated with student's willingness to go for higher education. Our analysis does not ignore such factors. In fact, even after taking the red-lining effect into the consideration, our method is able to identify two groups with bias.

\end{document}